%% file: neurips_2026.tex
\documentclass{article}

\PassOptionsToPackage{square,numbers,sort&compress}{natbib}

\usepackage{subfigure}

\usepackage[main, final]{neurips_2026}

\title{Estimation of the Label-Noise Transition Matrix with Performance Guarantees via Selective Classification}

\author{%
  Xabier~de~Juan$^{1}$ \quad Santiago~Mazuelas$^{1, 2}$ \quad Yilun Zhu$^3$ \quad Clayton Scott$^3$\\
  $^1$Basque Center of Applied Mathematics (BCAM)\\ $^{2}$IKERBASQUE-Basque Foundation for Science\\ $^3$Electrical and Computer Engineering, University of Michigan\\
  \texttt{\{xdejuan, smazuelas\}@bcamath.org} \quad \texttt{\{allanzhu, clayscot\}@umich.edu}
}

\input{preamble}

\begin{document}

\maketitle

\begin{abstract}
    Modern machine learning depends heavily on massive datasets, but obtaining high-quality annotations at scale is often expensive.
    As a result, learning from noisily-labeled data has become common, making accurate estimation of the label-noise transition matrix crucial.
    However, existing transition matrix estimators rely on the fragile estimation of class-posteriors and do not provide finite-sample performance guarantees.
    In this work, we propose a novel methodology to estimate the transition matrix based on one-sided selective classification.
    This approach bypasses class-posterior estimation, provides finite-sample performance guarantees, and leverages flexible learning methods for binary classification.
    Moreover, we introduce effective algorithms to implement the proposed methodology and provide their refined finite-sample performance bounds.
\end{abstract}


\input{main}

\newpage

\input{apendice}



\end{document}

%% file: preamble.tex
\usepackage[american]{babel}
\usepackage[utf8]{inputenc} 
\usepackage[T1]{fontenc}    
\usepackage{url}
\usepackage{xcolor}
\definecolor{newcolor}{rgb}{.8,.349,.1}
\usepackage{graphicx}
\usepackage{psfrag}
\usepackage{float}
\usepackage{amsmath,amsthm,amsfonts,amssymb,fancyhdr,bm}
\usepackage{mathtools}
\usepackage{enumerate,verbatim}
\usepackage{tablefootnote}
\usepackage{algorithm}
\usepackage[noend]{algpseudocode}

\usepackage{todonotes}

\usepackage{hyperref}       
\usepackage{booktabs}       
\usepackage{nicefrac}       
\usepackage{microtype}      

\usepackage{comment}

 \usepackage{setspace}
\let\Algorithm\algorithm
\renewcommand\algorithm[1][]{\Algorithm[#1]\setstretch{1.3}}

\usepackage{eucal}

\usepackage[skip=10pt plus1pt]{parskip}

\usepackage{hyperref}
\hypersetup{
    colorlinks,
    linkcolor={black},
    citecolor={black},
    urlcolor={black}
}

\usepackage[capitalize,noabbrev]{cleveref}

\newcommand{\R}{\mathbb{R}}
    \newcommand{\E}{\mathbb{E}}
    \renewcommand{\P}{\mathbb{P}}
\newcommand{\I}{\mathbb{I}}

\newcommand{\VC}{\mathsf{VC}}
\newcommand{\T}{\textbf{T}}
\newcommand{\st}{\, \up{ s.t. } \, }
\newcommand{\accept}{\texttt{\upshape{A}}}
\newcommand{\reject}{\texttt{\upshape{R}}}

\newcommand\up[1] {\mathrm{#1}}
\newcommand\set[1] {\mathcal{#1}}

\theoremstyle{definition}
\newtheorem{theorem}{Theorem}

\newtheorem{proposition}{Proposition}
\newtheorem{lemma}{Lemma}

\newtheorem{definition}{Definition}

\theoremstyle{remark}

\newtheorem*{remark*}{Remark}

\DeclareMathOperator{\sign}{sign}
\DeclareMathOperator{\opt}{opt}

\DeclareMathOperator{\temp}{temp}
\DeclareMathOperator{\anchor}{anchor}
\DeclareMathOperator{\noisy}{noisy}

\DeclarePairedDelimiterX{\infdivx}[2]{(}{)}{%
  #1\;\delimsize\|\;#2%
}

\let\epsilon=\varepsilon

\DeclarePairedDelimiter\floor{\lfloor}{\rfloor}
\newcommand{\norm}[1]{\left\lVert#1\right\rVert}
\usepackage{multirow}

\usepackage{csquotes}


%% file: main.tex
\section{Introduction}
The success of modern machine learning relies heavily on the availability of massive labeled datasets. 
However, obtaining high-quality annotations at scale is often cost-prohibitive, leading to the widespread adoption of cost-effective and less accurate labeling procedures \cite{song23,xiao15,deng09}.
Such compromise introduces label noise, where observed labels may differ from the underlying ground truth.
The probabilities of these label flips are captured by the label-noise transition matrix.

An accurate estimate of the transition matrix can address many of the problems caused by label noise. 
For instance, loss correction techniques leverage the transition matrix to recover the Bayes-optimal classifier from noisy samples \cite{natarajan13,scott15,blanchard16,liu16,patrini17,filippozzi24}. 
Beyond standard classification, a reliable estimate of the transition matrix can enable the construction of informative prediction sets for conformal prediction \cite{sesia24}, the implementation of fair classification models under biased data \cite{zhang25}, and the training of conditional diffusion models with noisy labels \cite{na2024labelnoise}.

Existing transition matrix estimators heavily rely on the pointwise estimation of the noisy class-posterior \cite{liu16,patrini17,xia19,xia20,zhang21,li21,mzhang21}.
This requirement is fragile since even a poor estimate at a single point can lead to a large subsequent estimation error for the transition matrix \cite{yong23}.
In addition, estimating the class-posterior is provably hard. 
For example, any estimator of the class-posterior suffers from the curse of dimensionality \cite{gyorfi02,lei14}, and common methods often result in inaccurate class-posterior estimations~\cite{guo17}.

While recent works have established consistency for certain estimators of the label-noise transition matrix \cite{li21,zhang21}, the literature lacks finite-sample performance guarantees.
Approaches in the closely related problem of mixture proportion estimation (MPE) \cite{blanchard14,scott15,blanchard16,ramaswamy16,katz20,saurabh21,zhu23} offer finite-sample performance guarantees.
However, most of these approaches are restricted to binary settings, while others rely on intractable algorithms.

In this work, we introduce a methodology for estimating the label-noise transition matrix based on one-sided selective classification.
This approach does not require pointwise class-posterior estimation and provides finite-sample performance guarantees.
Moreover, our methodology allows us to estimate the transition matrix by adapting established techniques from selective classification.
Specifically, the main contributions in the paper are as follows:
\begin{itemize}
    \setlength\itemsep{-0.7em}
    \item We frame the estimation of each column of the transition matrix as a one-sided selective classification problem, where we minimize the false discovery rate subject to a minimum coverage constraint.
    
    \item We provide finite-sample error bounds for the proposed approach, and prove that our estimator achieves the parametric convergence rate up to a small bias.
    In particular, we show that our methodology avoids the curse of dimensionality, unlike existing estimators.

    \item We propose computationally tractable algorithms for our methodology that leverage general methods for binary classification.
    
    \item We provide a theoretical analysis of the proposed algorithms, showing that they admit refined finite-sample error bounds analogous to classical generalization bounds for classification.
\end{itemize}


\section{Preliminaries}
\label{sec:preliminaries}

This section defines the label-noise transition matrix and describes existing methods for its estimation.

\subsection{Label-noise transition matrix}
Let $\set{X}\subseteq\R^d$ be the feature space and let $\set{Y}=\{1,2,\ldots,|\set{Y}|\}$ be the label set.
Instance and clean-label pairs $(X,Y)\in\set{X}\times\set{Y}$ follow an underlying distribution $\up{p}$, while their noisy-label counterparts $(X,\widetilde Y)\in\set{X}\times\set{Y}$ are drawn from the distribution $\widetilde{\up{p}}$.
We assume that the distribution of $\widetilde Y$ depends on $Y$, but not on $X$, i.e., $X$ and $\widetilde Y$ are independent given $Y$. 
This assumption is referred to as class-dependent label noise and it is standard in the literature (see e.g., \cite{liu16,patrini17,xia19,xia20,zhang21,li21,mzhang21,yong23}).
It has shown to be effective in practice \cite{li21,yong23}, even in scenarios where label noise may change across instances.

The label-noise transition matrix $\T\in [0,1]^{|\set{Y}|\times |\set{Y}|}$ is defined entry-wise as ${T}_{i,j} = \P(\widetilde Y = i \mid Y = j)$, for $i,j\in\set{Y}$, so that $T_{i,j}$ is the probability that an instance with true label $j$ is observed as noisy label $i$.
We assume that $\T$ is a row-diagonally dominant matrix (i.e., $T_{i,i} > T_{i,j}$ for all $j\neq i$), paralleling common assumptions in the literature \cite{patrini17,xia19,zhang21,li21,zhu24}.

The goal is to estimate the matrix $\T$ using only noisy samples drawn from $\widetilde{\up{p}}$.
Specifically, we assume that we have access to $n$ i.i.d.\ samples from $\widetilde{\up{p}}$, $\{ (x_k,\widetilde y_k)  \}_{k=1}^n\subseteq\set{X}\times\set{Y}$.

\subsection{Related work}
\label{sec:related_work}
The most common approach to estimate the transition matrix is based on identifying anchor points for each class \cite{liu16,patrini17}.
This method exploits the fact that clean and noisy class-posterior probabilities are related by the matrix $\T$, that is, \mbox{${\boldsymbol{\eta}}^{\noisy}(x) = \T \boldsymbol{\eta}(x)$} where $\eta^{\, \noisy}_i(x) = \P(\widetilde Y = i \mid X = x)$ and \mbox{$\eta_i(x) = \P( Y = i \mid X = x)$}.

The method relies on the \emph{anchor-point assumption} \cite{liu16,patrini17} which states that there are instances with class-posterior probability equal to one for each class.
In particular, the anchor-based method exploits the fact that if $x^{(j)}$ is an anchor point for class $j$, the entries of the transition matrix are given by $T_{i,j}=\eta_i^{\, \noisy}(x^{(j)})$.

The anchor-based method \cite{patrini17} first obtains an estimate $\widehat{\boldsymbol{\eta}}^{\noisy}(x)\in[0,1]^{|\set{Y}|}$ of the vector of noisy class-posteriors ${\boldsymbol{\eta}}^{\noisy}(x)\in[0,1]^{|\set{Y}|}$.
Then, for each class $j \in \set{Y}$, the algorithm estimates the correspon\-ding column of $\T$ through the following procedure.
\begin{enumerate}
    \setlength\itemsep{-0.7em}
    \item Find a candidate for anchor point of the $j$-th class by obtaining the instance that maximizes the $j$-th component of $\widehat{\boldsymbol{\eta}}^{\noisy}$, $x^{(j)} \in \arg\max_{x\in\set{X}}\widehat{{\eta}}^{\, \noisy}_j(x)$.

    \item Estimate the $(i,j)$-th entry of $\T$ as $\widehat{T}_{i,j}^{\anchor}(x^{(j)}) = \widehat{\eta}^{\, \noisy}_i(x^{(j)})$.
\end{enumerate}

Subsequent methods eliminated the explicit search for anchor points \cite{xia19,xia20,zhang21,li21,mzhang21}, yet they still critically rely on an accurate pointwise estimation of the class-posterior.
This dependency creates a severe vulnerability since a poorly estimated class-posterior at a single point can produce a large estimation error for the transition matrix \cite{yong23}.
Furthermore, common methods (e.g., neural networks) are known to yield highly unreliable class-posterior estimates (see e.g., \cite{guo17}).
Specifically, if $\alpha$ is the smoothness parameter of the class-posterior and $d$ is the dimension of the feature space, the minimax optimal pointwise estimation error of the class-posterior is $n^{-\alpha/(2\alpha +d)}$ \cite{gyorfi02,lei14}, yielding a prohibitively slow convergence rate, even in moderate dimensions.
The recent work \cite{yong23} bypasses the need for class-posterior estimation and instead relies on the minimization of noise-robust losses.
However, this approach relies on multiple assumptions that are 
unlikely to hold in practice.
For instance, Assumption 1 in \cite{yong23} requires that a minimizer of cross-entropy is the only minimizer of the noise-robust loss, a condition that is often not satisfied (see e.g., \cite{zhang18,long22}).

Existing literature does not provide finite-sample performance guarantees and solely consistency has been established for some estimators \cite{li21,zhang21}.
Estimators in the closely related field of MPE \cite{blanchard14,scott15,blanchard16,ramaswamy16,katz20,saurabh21,zhu23} instead estimate the inverse flip-rates $\P(Y=i \mid \widetilde Y = j)$, and offer attractive finite-sample performance guarantees that decay at the parametric rate of $O(1/\sqrt{n})$ up to a small bias term.
However, the theoretical elegance of these MPE estimators comes at the cost of severe practical limitations.
For instance, the estimator in \cite{scott15,blanchard16} is based on identifying subsets of the feature space using an exhaustive search that relies on VC dimension bounds, which are known to be loose \cite{scott15}.
Alternatively, the estimators in \cite{ramaswamy16,saurabh21} are limited to the binary case.

In this work, we propose to estimate the transition matrix using methods for one-sided selective classification.
Standard selective classification (or classification with a reject option) \cite{el-yaniv10,geifman17,bartlett08,cortes16} allows models to abstain when uncertain.
This is a crucial capability when errors are significantly more expensive than abstaining (e.g., in medical applications \cite{hanczar08,swaminathan24}).
Formally, the goal is to learn a standard classifier $f$ and a selection function $h \colon \set{X} \to \{\accept,\reject\}$ \cite{el-yaniv10} that together minimize the probability of misclassification on the accepted samples, $\P( Y \neq f(X)  \mid h(X) = \accept )$, subject to a minimum coverage constraint for the selection function $\P( h(X) = \accept)\geq \gamma$, where $\gamma\in(0,1)$.

One-sided selective classification \cite{gangrade21} simplifies the standard setting by fixing the prediction to a target class $j \in \set{Y}$ (i.e., $f(x)=j$, $\forall x \in \set{X}$).
This problem arises in scenarios where it is necessary to isolate a high-purity subset of a single class, for instance to perform high-confidence auditing.
The goal then reduces to learning a selection function $h$ that minimizes the false discovery rate $\P( Y \neq j \mid h(X) = \accept )$ for the class $y=j$ under the coverage constraint $\P( h(X) = \accept)\geq \gamma$.

\section{Transition matrix estimation via one-sided selective classification}
\label{sec:classification}
In this section, we present a novel methodology for estimating the transition matrix based on one-sided selective classification.
For each class $j\in\set{Y}$, we estimate the corresponding column of $\T$ through the following procedure.
\begin{enumerate}
    \item Learn a selection function $h^{(j)}\colon \set{X} \to \{ \accept,\reject\}$ which minimizes the false discovery rate for the noisy label $\widetilde y = j$, subject to a minimum coverage constraint.
    
    \item Estimate the $j$-th column of $\T$ by computing the empirical label probabilities of instances accepted by $h^{(j)}$.
\end{enumerate} 

Formally, we aim to solve for each class $j\in \set{Y}$ the following problem
\begin{equation}
    \begin{aligned}
            \widetilde R_\gamma^{(j)}=\min_{h} \quad & \P(\widetilde Y \neq j \mid h(X) = \accept)\\
        \textrm{s.t.} \quad & \P(h(X) = \accept)\geq\gamma
    \end{aligned}
    \label{eq:selective_classification}
\end{equation}
where $\gamma\in(0,1)$ is the minimum coverage constraint and the optimization is over all measurable selection functions $h\colon \set{X} \to \{ \accept,\reject\}$.
The smallest possible false discovery rate subject to the minimum coverage constraint is denoted by $\widetilde R_\gamma^{(j)}$.

The estimate $\widehat\T$ is obtained using the selection functions and noisy samples as follows.
\begin{definition}
    Let $h^{(j)} \colon \set{X}\to \{  \accept,\reject \}$ be a selection function for the $j$-th label, and let $S$ be a subset of $\{1,2,\ldots, n\}$ where $m=|S|$.
    For every $i\in\set{Y}$ define
    \begin{align}
        \widehat T_{i,j}(h^{(j)};S)  = \frac{\sum_{k\in S}\I\{ h^{(j)}(x_k) = \accept, \widetilde y_k = i  \} }{ \#_m( h^{(j)} ) }  \label{eq:def_estimator}
    \end{align}
    where $\#_m(h^{(j)}) = \sum_{k\in S}\I \{  h^{(j)}(x_k) = \accept \}$.
    When the set of indices $S$ is clear from the context we denote the estimator as $\widehat T_{i,j}(h^{(j)})$.
\end{definition}

The estimator $\widehat T_{i,j}(h^{(j)};S)$ is well suited to estimating the transition matrix.
In particular, we have $\widehat T_{i,j}(h^{(j)};S) \approx T_{i,j}$ for a near-optimal solution $h^{(j)}$ of \eqref{eq:selective_classification}.
This holds because $\widehat T_{i,j}(h^{(j)};S)$ is the empirical version of $\P(\widetilde Y = i \mid h^{(j)}(X) = \accept)$, which is close to $\P(\widetilde Y = i \mid Y = j)$ for a selection function $h^{(j)}$ with small false discovery rate $\P(\widetilde Y \neq j \mid h^{(j)}(X) = \accept)$.

The next theorem shows that the error of the proposed estimator decreases as
\begin{align*}
    |T_{i,j}-\widehat T_{i,j}(h^{(j)}) | \lesssim \epsilon_{\text{bias}} + \frac{1}{\sqrt{\gamma n}}
\end{align*}
where $\gamma$ is the minimum coverage constraint, and $\epsilon_{\text{bias}}$ is the bias incurred due to the false discovery rate of the selection function.
\vspace{1.5em}
\begin{theorem}
\label{thm:performance_guarantees}
    Let $h^{(j)} : \set{X} \to \{\accept,\reject\}$ be a selection function for the $j$-th label with \mbox{$\P( h^{(j)} (X) =\accept) \geq \gamma$}, and $\epsilon_{\opt}(h^{(j)})$ be its excess false discovery rate, that is
    \begin{align*}
    \epsilon_{\opt}(h^{(j)}) = \P(\widetilde Y \neq j \mid  h^{(j)}(X) = \accept)-\widetilde R_\gamma^{(j)}.
    \end{align*}
    If $\widehat T_{i,j}(h^{(j)};S)$ for $|S|=m$ is the estimator in \eqref{eq:def_estimator}, then 
    \begin{align}
        \max_{i\in\set{Y}}|T_{i,j}-\widehat T_{i,j}(h^{(j)};S) | \leq C_{T}\big( R_{\gamma}^{(j)} + \epsilon_{\opt}(h^{(j)}) \big) + \sqrt{\frac{2\log(|\set{Y}|/\delta)}{\#_m(h^{(j)})}}
        \label{eq:performance_guarantees}
    \end{align}
    holds with probability at least $1-\delta$ over the draw of the samples indexed by $S$, where \mbox{$C_{T} = 2/(T_{j,j}-\max_{l\neq j} T_{j,l})$} and $R_{\gamma}^{(j)} = \min\limits_h \{ \P(Y \neq j \mid h(X) = \accept) \st \P(h(X) =\accept )\geq \gamma\}$. 
\end{theorem}
\vspace{-1.5em}
\begin{proof}
    See Appendix~\ref{proof:thm:performance_guarantees}.
\end{proof}

The theorem above provides finite-sample performance guarantees for the proposed methodology that show a bias-variance decomposition for the error.

The bias term is given by $C_{T}\big( R_{\gamma}^{(j)} + \epsilon_{\opt}(h^{(j)}) \big)$, where the constant $C_T$ acts as a condition number for the transition matrix $\T$.
This constant is inversely proportional to the margin $T_{j,j}-\max_{l\neq j} T_{j,l}$, meaning that the estimation problem is better conditioned when the diagonal dominance of $\T$ is more pronounced.
The bound also incorporates a suboptimality gap, $\epsilon_{\opt}$, that accounts for the difficulty of finding the optimal selection function in \eqref{eq:selective_classification} with finite data.

The term $R_\gamma^{(j)}$ is the smallest false discovery rate on clean data. 
As discussed at the end of this section, such smallest error is small as long as a relaxed anchor-point assumption is satisfied. 
In particular, the anchor-point assumption corresponds to the case where $R_\gamma^{(j)}\to 0$ as $\gamma$ tends \mbox{to 0}. 
Therefore, Theorem~\ref{thm:performance_guarantees} also provides performance guarantees for cases where the anchor-point assumption is not satisfied (i.e.,  $\lim_{\gamma\to0} R_\gamma^{(j)}> 0$). 
In these cases, the limit value represents an unavoidable positive bias, while the proposed methodology remains effective as long as $R_\gamma^{(j)}$ takes small values.

The variance term $\sqrt{{2\log(|\set{Y}|/\delta)}/{\#_m(h^{(j)})}}$ scales at the parametric rate of $O(1/\sqrt{n})$.
Specifically, the minimum coverage constraint, $\P( h^{(j)} (X) =\accept) \geq \gamma$, ensures that the number of accepted samples satisfies $\#_m(h^{(j)}) \gtrsim \gamma n$ with high probability.

The minimum coverage $\gamma$ controls the trade-off between the bias and the variance terms described above.
Decreasing $\gamma$ reduces the value of $R_\gamma^{(j)}$, as a selection function that is required to accept fewer instances can provide a smaller false discovery rate. 
However, this reduction incurs a cost, as rejecting more samples shrinks $\#_m(h^{(j)})$, thereby increasing the variance.
Choosing $\gamma=\omega(1/n)$ implies that both $R_\gamma^{(j)}$ and $1/\sqrt{\#_m(h^{(j)})}$ decrease with the number of samples $n$.
In particular, if $R_\gamma^{(j)}$ decreases with $\gamma$ as $O(\gamma^\rho)$, taking $\gamma=\Theta(n^{-1/(1+2\rho)})$ results in \mbox{$R_\gamma^{(j)}+ 1/\sqrt{\#_m(h^{(j)})} = O(n^{-1/2+1/(2+4\rho)})$}, which is close to the parametric rate when $\rho$ is large.

Existing theoretical results for MPE methods provide performance guarantees for estimates of the inverse flip-rates $\P(Y=i \mid \widetilde Y = j)$ that also reveal a bias-variance decomposition, with the variance term scaling as $O(1/\sqrt{n})$ \cite{ramaswamy16,saurabh21}.
However, these results are limited to the binary setting and to specific learning methods.
For instance, the estimator in \cite{saurabh21} is restricted to the level sets of a fixed scoring function, while \cite{ramaswamy16} is limited to kernel-based methods. 
In contrast, our approach addresses the estimation of transition matrices offering a more general methodology since it covers the multiclass setting and allows the usage of any method for one-sided selective classification.

To the best of our knowledge, the bounds in Theorem~\ref{thm:performance_guarantees} represent the first finite-sample performance guarantees for methods that estimate the label-noise transition matrix.
Using the same line of reasoning, the next result provides finite-sample guarantees for the anchor-based estimator.

\vspace{1.5em}
\begin{proposition}
\label{prop:anchors}
    Let $\widehat{\boldsymbol{\eta}}^{\noisy}$ be an estimate of the noisy class-posteriors ${\boldsymbol{\eta}}^{\noisy}$.
    For every class $j\in\set{Y}$, let $x^{(j)} \in \arg\max_{x\in\set{X}}\widehat{{\eta}}^{\, \noisy}_j(x)$, and define
    \begin{align*}
         \epsilon_{\opt}^{\anchor}( x^{(j)}) = \P( \widetilde Y \neq j \mid X =  x^{(j)})- \widetilde R_0^{(j)}
    \end{align*}
    where $\widetilde R_0^{(j)} = \min_{x\in\set{X}}\P( \widetilde Y \neq j \mid X = x) $.
    Then, the anchor-based estimator \mbox{$\widehat{T}_{i,j}^{\anchor}(x^{(j)}) = \widehat{\eta}^{\, \noisy}_i(x^{(j)})$} satisfies
    \begin{align}
        \max_{i\in\set{Y}}|T_{i,j}-\widehat T_{i,j}^{\anchor}(  x^{(j)} ) | \leq C_T (R_{0}^{(j)} +\epsilon_{\opt}^{\anchor}(x^{(j)}) )+ \max_{i\in\set{Y}}| {{\eta}}^{\, \noisy}_i(x^{(j)}) -\widehat{{\eta}}^{\, \noisy}_i(x^{(j)})|
        \label{eq:anchors}
    \end{align}
    where $R_0^{(j)} = \min_{x\in\set{X}} \P(Y \neq j \mid X = x)$.
\end{proposition}
\vspace{-1.5em}
\begin{proof}
    See Appendix~\ref{proof:prop:anchors}.
\end{proof}
A comparison between Theorem~\ref{thm:performance_guarantees} and Proposition~\ref{prop:anchors} indicates that the proposed methodology addresses the main limitation of the anchor-based method.
More precisely, Proposition~\ref{prop:anchors} shows that the anchor-based estimator suffers from the curse of dimensionality since it relies on an accurate class-posterior estimate.
In particular, there are scenarios where the last term of the upper bound in \eqref{eq:anchors} is exactly of the order $n^{-\alpha/(2\alpha +d)}$ \cite{gyorfi02,lei14}, which is a very slow rate for real-world high-dimensional datasets.

The value of $R_0^{(j)}$ quantifies deviations from the anchor-point assumption that corresponds to $R_0^{(j)}$ being equal to zero. 
The value $R_\gamma^{(j)}$ in Theorem~\ref{thm:performance_guarantees} above plays a similar role to $R_0^{(j)}$ but providing more detailed information about the hardness of estimating the transition matrix, as described above.
In particular, we have that $R_\gamma^{(j)}$ approaches $R_0^{(j)}$ when $\gamma\to0$, and the speed with which $R_\gamma^{(j)}$ decreases with $\gamma$ determines the best possible rate of decrease for the estimation error, as described above.

The suboptimality term for the anchor-based method $\epsilon_{\opt}^{\anchor}$ also depends strongly on the pointwise estimation error of the class-posterior since $x^{(j)}$ is obtained as the instance with the highest class-posterior estimate.
Therefore, there are scenarios where $\epsilon_{\opt}^{\anchor} \asymp n^{-\alpha/(2\alpha +d)}$.
In contrast, the suboptimality term for the presented methodology $\epsilon_{\opt}$ is the excess false discovery rate in a problem of selective classification that is small under realistic assumptions.
Specifically, the next section presents computationally tractable algorithms that provide selection functions with small suboptimality gap $\epsilon_{\opt}$.

\section{Tractable algorithms with finite-sample performance guarantees}
\label{sec:algorithms}
The learning methodology from Section~\ref{sec:classification} can be implemented utilizing general techniques for selective classification (also known as classification with rejection/abstention) \cite{el-yaniv10,geifman17,cortes16}.
In the following we present two effective approaches and provide their corresponding finite-sample performance guarantees.

\subsection{Threshold selection approach}
A prevalent approach within selective classification, classification with abstention, and classification with generalized metrics, is to rely on the confidence scores of a standard binary classifier \cite{koyejo14,geifman17}.
This approach can be adapted to our methodology for estimating the transition matrix as follows.
First, split the noisy samples into two disjoint sets.
Then, for each class $j\in\set{Y}$, the algorithm estimates the corresponding column of $\T$ through the following procedure.
\vspace{-0.9em}
\begin{enumerate}
\setlength\itemsep{-0.7em}
    \item Learn a binary scoring function $s^{(j)}\colon \set{X}\to\R$ to distinguish class $\widetilde Y = j$ from all others, using the samples in the first split.
    For every $\tau\in\R$, define a selection function as \mbox{$h^{{(j)}}_{\tau}(x) = \accept$} if $s^{(j)}(x)\geq \tau$, and $h^{{(j)}}_{\tau}(x) = \reject$ otherwise.

    \item Select the threshold $\widehat \tau$ that minimizes the empirical version of $\P(\widetilde Y \neq j\mid  h_{\tau}^{(j)}(X) = \accept)$ (i.e., that maximizes $\widehat T_{j,j}(h^{(j)}_{\tau})$) while satisfying an empirical minimum coverage constraint, using the samples in the second split.

    \item Estimate the $(i,j)$-th entry of $\T$ as $\widehat T_{i,j}( h_{\widehat \tau}^{(j)})$, using the samples in the second split.
\end{enumerate}
\vspace{-0.5em}
This procedure is formalized in Algorithm~\ref{alg:threshold} and theoretically analyzed in Theorem~\ref{thm:threshold}.

\vspace{-0.5em}

\begin{algorithm}[!h]
\small
\caption{Transition matrix estimation via threshold selection}\label{alg:threshold}
\begin{algorithmic}[1]
\Require $\{(x_k,\widetilde y_k)\}_{k=1}^n$ and lower bound for the number of accepted samples $N_+>0$.
\State Split $\{1,2,\ldots,n\}$ into two disjoint sets, ${S}_1$ (size $m_1$) and ${S}_2 $ (size $m_2\geq N_+$)
\For{$j\in\set{Y}$}
\State $T_{\temp}\gets 0$
\State \parbox[t]{\dimexpr\linewidth-\algorithmicindent}{Learn a binary score function $s^{(j)} \colon \set{X}\to \R$ to distinguish class $\widetilde Y = j$ from all others, using the samples indexed by $S_1$ \label{step:learn_score}}
\State $(\tau_r)_{r=1}^{m_2}\gets \texttt{Sort}(\{s^{(j)}(x_k)\}_{k\in S_2})$ sort the scores in descending order 
\For{$r = N_+,\ldots, m_2$}
\State Define $h^{(j)}_{\tau_r}(x) = \accept$ if $s^{(j)}(x)\geq \tau_r$, and $h^{(j)}_{\tau_r}(x) = \reject$ otherwise, for every $x\in\set{X}$
\If{$\widehat T_{j,j}(h^{(j)}_{\tau_r};S_2) \geq T_{\temp}$ }
    \State $\widehat\tau \gets \tau_r$
    \State $T_{\temp} \gets \widehat T_{j,j}(h^{(j)}_{\widehat \tau};S_2)$
        \State $\widehat T_{i,j}\gets \widehat T_{i,j}(h^{(j)}_{\widehat  \tau};S_2)$ for $i\in\set{Y}$
\EndIf
\EndFor
\EndFor\\
\Return $\widehat \T$
\end{algorithmic}
\end{algorithm}
\vspace{-0.5em}
In the next theorem, we show that the error of Algorithm~\ref{alg:threshold} depends on the ranking excess risk of the learned score function in  Step~\ref{step:learn_score}.

\begin{theorem}
    \label{thm:threshold}
    Let $\widehat \T$ be the output of Algorithm~\ref{alg:threshold} and fix $j\in\set{Y}$.
    Let \mbox{$\gamma = \P(h_{\widehat \tau}^{(j)}(X)=\accept)$} be the minimum coverage, that is $\gamma = \P(s^{(j)}(X) \geq \widehat\tau)$.

    \vspace{-0.8em}
    If $\widehat\tau_\eta$ satisfies \mbox{$\gamma=\P( \eta_j^{\, \noisy}(X) \geq \widehat\tau_\eta)=\P(s^{(j)}(X) \geq \widehat\tau)$}, and we take
    \begin{align*}
        \mathcal{E}_{\text{rank}} = \P(\widetilde Y \neq j,  s^{(j)}(X) \geq \widehat\tau,  \eta_j^{\, \noisy}(X) < \widehat\tau_\eta)-   \P(\widetilde Y \neq j, s^{(j)}(X) <\widehat\tau,  \eta_j^{\, \noisy}(X) \geq \widehat\tau_\eta)
    \end{align*}
    then
    \begin{align}
       \max_{i\in\set{Y}} |T_{i,j}-\widehat T_{i,j} | \leq C_T \left( R_{\gamma}^{(j)}+ \frac{\mathcal{E}_{\text{rank}}}{\gamma}  \right)+ \sqrt{\frac{1000(\log(8 m_2)+\log(4|\set{Y}|/\delta))}{\#_{m_2}(  h_{\widehat \tau}^{(j)}   )}}
    \end{align}
    holds with probability at least $1-\delta$ where $\#_{m_2}(  h_{\widehat \tau}^{(j)}   )$ is the number of accepted samples in the second split.
\end{theorem}
\vspace{-0.6em}
\begin{proof}
    See Appendix~\ref{proof:thm:threshold}
\end{proof}

The theorem above shows that the error of Algorithm~\ref{alg:threshold} is small when the learned score function is close to being order preserving with respect to the noisy class-posterior, i.e., $\mathcal{E}_{\text{rank}}$ is close to 0.
In particular, if the score is order preserving (i.e., for every pair of instances \mbox{$x, x'\in\set{X}$},  \mbox{$s^{(j)}(x) <  s^{(j)}(x') \implies \eta_j^{\, \noisy}(x)\leq\eta_j^{\, \noisy}(x')$}) we have that $\mathcal{E}_{\text{rank}}=0$.
This follows since in that case $\{ x \in \set{X} : s^{(j)}(x) < \widehat \tau \} \subseteq \{ x \in \set{X} : \eta^{\, \noisy}_j(x) < \widehat\tau_\eta \}$, and both sets have the same probability.
Notice also that the existence of $\widehat\tau_\eta$ is guaranteed if $\eta_j^{\, \noisy}(X)$ is continuously distributed.

The requirement of a low ranking excess risk $\mathcal{E}_{\text{rank}}$ is milder than assuming an accurate class-posterior estimation since commonly used classifiers (e.g., SVMs or boosted trees) are more effective at producing rankings than precise posterior probabilities \cite{caruana06}.

The ranking excess risk $\mathcal{E}_{\text{rank}}$ decreases as $m_1$ grows when the score function is learned by empirical risk minimization over a hypothesis class $\set{F}$.
Specifically, results for the bipartite ranking problem in \cite{clemencon08} can be used to show that the ranking excess risk satisfies
\vspace{0.5em}
\begin{align*}
        \mathcal{E}_{\text{rank}} \lesssim   \set{A}_{\text{rank}}(\set{F})+\sqrt{\frac{\VC{(\set{F})}\log(m_1)}{m_1}}
\end{align*}%
\vspace{0.5em}%
where $\set{A}_{\text{rank}}(\set{F})$ is the approximation error of the hypothesis class $\set{F}$ for the ranking loss in \cite{clemencon08} (difference between the best in class risk and the Bayes risk).

The hyperparameter $N_+$ serves to control the minimum coverage $\gamma$, which is of order $N_+/{m_2}$.
Specifically, by Hoeffding's inequality we have that $\gamma \geq {N_+}/{m_2} - O(1/\sqrt{m_2})$ with high probability.
In practice, choosing $N_+$ involves a trade-off since it should be large enough to accept a sufficient number of instances (to reduce the variance), while remaining small enough to maintain a low false discovery rate (to minimize the bias). 
We defer a detailed discussion on the effect of $N_+$ to Appendix~\ref{app:algorithms}.

In the following we present another algorithm for estimating $\T$ that satisfies further refined finite-sample guarantees.
\vspace{0.5em}
\subsection{Cost-sensitive risk minimization approach}
A common approach in selective classification involves penalizing misclassifications and assigning a tunable cost to rejections \cite{cortes16,cortes24}.
This idea can be adapted to our methodology for estimating the transition matrix as follows.
First, split the noisy samples into two disjoint sets.
Then, for each class $j\in\set{Y}$, the algorithm estimates the corresponding column of $\T$ through the following procedure.
\begin{enumerate}
    \item Learn a binary classifier $h^{(j)}_c \colon\set{X}\to\{\accept,\reject\}$ where $\accept$ targets $\widetilde Y = j$ and $\reject$ targets $\widetilde Y \neq j$, by penalizing the misclassification of the class $\widetilde Y = j$ and assigning a cost $c\in(0,1)$ to all instances classified as $\reject$.
    That is, we learn $h^{(j)}_c$ by minimizing the loss
    \vspace{0.25em}
    \begin{align}
        \ell_c^{(j)}(h(x), \widetilde y) = \I\{ \widetilde  y \neq j , h(x) = \accept \} + c \I\{ h(x) = \reject \}.
        \label{eq:ell_c}
    \end{align}
    \item Select the cost $\widehat c$ that minimizes the empirical version of $\P(\widetilde Y \neq j\mid  h_c^{(j)}(X) = \accept)$ (i.e., that maximizes $\widehat T_{j,j}( h_{c}^{(j)})$) while satisfying an empirical minimum coverage constraint, using the samples in the second split.

    \item Estimate the $(i,j)$-th entry of $\T$ as $\widehat T_{i,j}( h_{\widehat c}^{(j)})$, using the samples in the second split.
\end{enumerate}
\vspace{0.2em}
This procedure is formalized in Algorithm~\ref{alg:cost_dependent} and theoretically analyzed in Theorem~\ref{thm:cost_dependent}, where we prove that the error of Algorithm~\ref{alg:cost_dependent} depends on the (classification) excess risk of the solution to the cost-sensitive binary classification problem in Step~\ref{step:cost_sensitive_optimization} of Algorithm~\ref{alg:cost_dependent}. 
    \vspace{0.5em}
    
    \begin{algorithm}[!h]
    \small
    \caption{Transition matrix estimation via cost-sensitive risk minimization}\label{alg:cost_dependent}
    \begin{algorithmic}[1]
    \Require $\{(x_k,\widetilde y_k)\}_{k=1}^n$, grid spacing $\epsilon$ and lower bound for the number of accepted samples $N_+>0$.
    \State Split $\{1,2,\ldots,n\}$ into two disjoint sets, ${S}_1$ (size $m_1$) and ${S}_2 $ (size $m_2\geq N_+$)
    \For{$j\in\set{Y}$}
    \State $T_{\temp}\gets 0$
    \For{$q = 1,2,\ldots, \floor{1/\epsilon}$}
    \State $c\gets \epsilon q$
    \State \parbox[t]{\dimexpr\linewidth-\algorithmicindent}{Learn a cost-sensitive binary classifier $h^{(j)}_c$ with costs given by $c$ as in \eqref{eq:ell_c}, using the samples\\ indexed by $S_1$ \label{step:cost_sensitive_optimization}}
    \If{$ \#_{m_2}( h_c^{(j)} ) \geq N_+ \text{ \textbf{and} }\widehat T_{j,j}( h_{c}^{(j)};S_2) \geq T_{\temp}$ } 
        \State $\widehat c \gets c$
        \State $T_{\temp} \gets \widehat T_{j,j}(h_{\widehat c}^{(j)};S_2)$
            \State $\widehat T_{i,j}\gets \widehat T_{i,j}(h_{\widehat c}^{(j)};S_2)$ for $i\in\set{Y}$
    \EndIf
    \EndFor
    \EndFor\\
    \Return $\widehat \T$
    \end{algorithmic}
    \end{algorithm}
    

    \begin{theorem}
    \label{thm:cost_dependent}
    Let $\widehat \T$ be the output of Algorithm~\ref{alg:cost_dependent} and fix $j\in\set{Y}$.
    Let $\gamma = \P(h_{\widehat c}^{(j)}(X)= \accept)$ be the minimum coverage. 
    If $h_{\widehat c}^{(j)}$ is the classifier corresponding with the output for column $j$ and we take
    \begin{align*}
        \mathcal{E}_{\text{cost}}=\E[\ell_{\widehat c}^{(j)}( h_{\widehat c}^{(j)}(X),\widetilde Y)] - \inf_{h}\E[\ell_{\widehat c}^{(j)}( h(X),\widetilde Y)]
    \end{align*}
    
    \vspace{-1.5em}
    then,
        \vspace{-0.7em}
        \begin{align}
           \max_{i\in\set{Y}} |T_{i,j}-\widehat T_{i,j} | \leq C_T\left( R_{\gamma}^{(j)} +  \frac{\mathcal{E}_{\text{cost}}}{\gamma} \right) + \sqrt{\frac{2\log(|\set{Y}|/(\epsilon\delta))}{\#_{m_2}( h_{\widehat c}^{(j)})}}
        \end{align}

        \vspace{-1.2em}
        holds with probability at least $1-\delta$ where $\#_{m_2}( h_{\widehat c}^{(j)})$ is the number of accepted samples in the second split.
    \end{theorem}
    \vspace{-1.2em}
    \begin{proof}
        See Appendix~\ref{proof:thm:cost_dependent}.
    \end{proof}
    \vspace{-0.9em}
    
    Theorem~\ref{thm:cost_dependent} shows that the error of Algorithm~\ref{alg:cost_dependent} is small, provided that the binary cost-sensitive classification problem in Step~\ref{step:cost_sensitive_optimization} is solved with a small excess risk $\mathcal{E}_\text{cost}$.
    This requirement is significantly milder than that of a pointwise accurate class-posterior estimate, and the small ranking excess risk required by the thresholding approach in Theorem~\ref{thm:threshold}.
    In particular, achieving a small excess risk corresponds to an accurate classification of instances, whereas the thresholding approach additionally requires correct ranking of instances' probabilities.

    \vspace{-0.7em}
    In the next theorem, we provide performance guarantees for Algorithm~\ref{alg:cost_dependent} in terms of the excess risk of a convex surrogate loss and the $\psi$-transform of the surrogate from \cite[Def.~2]{bartlett06}. 
    Specifically, let $\Phi$ be a classification-calibrated convex surrogate of the 0-1 loss $t\mapsto\I\{t < 0\}$ such as the hinge loss ($\Phi(t)= \max\{0,1-t\}$) or the logistic loss ($\Phi(t) = \log(1+e^{-t})$) \cite[Def.~1]{bartlett06},
    we define the convex surrogate of $\ell_c^{(j)}$ as
    \begin{align*}
        L_c^{(j)}( f(x),\widetilde y) = \I\{ \widetilde y \neq j\} \Phi(- f(x)) + c \Phi(f(x))
    \end{align*}
    where $f$ is a binary score function.

    \begin{theorem}
    \label{thm:cost_dependent_surrogate}
    Let $\widehat \T$ be the output of Algorithm~\ref{alg:cost_dependent} assuming that the classifiers learned in Step~\ref{step:cost_sensitive_optimization} take the form $h_c^{(j)}(x) = \sign(f_c^{(j)}(x))$ for $x\in\set{X}$ with $f_c^{(j)}$ a binary score function.
    Fix $j\in\set{Y}$, let $\gamma = \P(h_{\widehat c}^{(j)}(X)= \accept)$ be the minimum coverage and let
    \vspace{-0.4em}
    \begin{align*}
        \mathcal{E}_\Phi=\E[L_{\widehat c}^{(j)}( f_{\widehat c}^{(j)}(X),\widetilde Y)] - \inf_{f}\E[L_{\widehat c}^{(j)}( f(X),\widetilde Y)]
    \end{align*}

    \vspace{-1.5em}
    be the excess risk of the convex surrogate loss.
    Then 
    \vspace{-0.7em}
        \begin{align}
           \max_{i\in\set{Y}} |T_{i,j}-\widehat T_{i,j} | \leq C_T\left( R_{\gamma}^{(j)} +  \frac{2\psi_\Phi^{-1}\left(\mathcal{E}_\Phi/2\right)}{\gamma} \right) + \sqrt{\frac{2\log(|\set{Y}|/(\epsilon\delta))}{\#_{m_2}( h_{\widehat c}^{(j)})}}
        \end{align}
        
        \vspace{-1.7em}
        holds with probability at least $1-\delta$.
    \end{theorem}
    \vspace{-1.3em}
    \begin{proof}
        See Appendix~\ref{proof:thm:cost_dependent_surrogate}.
    \end{proof}
    
    Theorem~\ref{thm:cost_dependent_surrogate} guarantees that the error of Algorithm~\ref{alg:cost_dependent} remains small as long as the learning technique used in Step~\ref{step:cost_sensitive_optimization} successfully minimizes a classification-calibrated convex surrogate loss.
    Moreover, reducing the excess risk of the convex surrogate loss to zero makes the suboptimality gap ${2\psi_\Phi^{-1}\left(\mathcal{E}_\Phi/2\right)}/\gamma$ converge to zero since the $\psi_\Phi$-transform satisfies $\psi_\Phi^{-1}(t)\to0$ when $t\to0$ as shown in \cite[Thm.~1]{bartlett06}.
    \vspace{0.2em}

    In practice, the learning in Step~\ref{step:cost_sensitive_optimization} is performed by minimizing the empirical risk of the convex surrogate loss $L_c^{(j)}$ over a hypothesis class $\set{F}$. 
    By standard generalization bounds \cite{mohri18}, we have $\mathcal{E}_\Phi \lesssim \set{A}_{\text{surrogate}}(\set{F})+O(\sqrt{{\VC(\set{F})\log(m_1/\VC(\set{F})) }/{m_1}})$
    where \mbox{$\set{A}_{\text{surrogate}}(\set{F})=\inf_{f\in \set{F}}\E[L_{\widehat c}^{(j)}( f(X),\widetilde Y)]- \inf_{f}\E[L_{\widehat c}^{(j)}( f(X),\widetilde Y)]$} is the approximation error of the hypothesis class $\set{F}$ for the convex surrogate loss $L_{\widehat c}^{(j)}$.
    \vspace{0.2em}

    The choice of the convex surrogate $\Phi$ specifies the $\psi_\Phi^{-1}$-transform, which determines the convergence rate of the term ${\psi_\Phi^{-1}\left(\mathcal{E}_\Phi/2\right)}$.
    For the hinge loss ($\Phi(t)= \max\{0,1-t\}$), $\psi_\Phi$ is simply the identity function \cite{bartlett06}. Therefore, the excess risk bound scales linearly, yielding
    \vspace{0.35em}
    \begin{align*}
        \max_{i\in\set{Y}} |T_{i,j}-\widehat T_{i,j} | \lesssim  R_{\gamma}^{(j)} + \set{A}_{\text{surrogate}}(\set{F})+\sqrt{\frac{\VC{(\set{F})}\log(m_1/\VC{(\set{F}))}}{m_1}} + \sqrt{\frac{\log(|\set{Y}|/\epsilon)}{m_2}}.
    \end{align*}
    \vspace{0.15em}
    
    For the logistic loss we have $\psi_\Phi^{-1}(t)\lesssim\sqrt{t}$ \cite{bartlett06}, which introduces a square-root dependence on the excess risk resulting in a slightly slower convergence rate.    
    \vspace{0.2em}

The results above show that the methodology presented can be implemented using flexible and effective techniques for selective classification.
In particular, the theoretical results show that the proposed approach can result in significantly improved error rates that quickly decrease with the number of samples. In the following, we provide numerical results that illustrate such improved performance using real-world datasets.
\vspace{0.3em}

\subsection{Experimental results}
\label{sec:experiments}
\vspace{0.1em}
In what follows, we provide experimental results that illustrate the theoretical results presented above.
Specifically, we show that the proposed algorithms achieve significantly lower error rates than methods that rely on pointwise estimates of class-posteriors and that the error rates of the proposed algorithms converge at an order comparable to that of a strong oracle.
Furthermore, in \cref{app:tables} we provide additional experimental results on running times and dimensionality's impact.
The code implementing the methods presented and reproducing the experiments can be found at \url{https://github.com/MachineLearningBCAM/label-noise-T-estimation-NeurIPS2026}.

\vspace{0.1em}

\begin{figure*}[h]
\centering
\subfigure[Letter dataset]{%
\label{fig:letter}%
\psfrag{0.05}[cc][cc]{\scalebox{0.7}{$0.05$\;}}%
\psfrag{0.04}[cc][cc]{\scalebox{0.7}{$0.04$\;}}%
\psfrag{0.03}[cc][cc]{\scalebox{0.7}{$0.03$\;}}%
\psfrag{0.02}[cc][cc]{\scalebox{0.7}{$0.02$\;}}%
\psfrag{0.01}[cc][cc]{\scalebox{0.7}{$0.01$\;}}%
\psfrag{0.005}[cc][cc]{\scalebox{0.7}{$0.005$ \;}}%
\psfrag{0.003}[cc][cc]{\scalebox{0.7}{$0.003$ \; }}%
\psfrag{0.001}[cc][tc]{\scalebox{0.7}{$0.001$ \; }}%
\psfrag{1000}[cc][cc]{\scalebox{0.7}{$1000$}}%
\psfrag{2500}[cc][cc]{\scalebox{0.7}{$2500$}}%
\psfrag{5000}[cc][cc]{\scalebox{0.7}{$5000$}}%
\psfrag{10000}[cc][cc]{\scalebox{0.7}{$10000$}}%
\psfrag{15000}[cc][cc]{\scalebox{0.7}{$15000$}}%
\psfrag{17500}[cc][cc]{\scalebox{0.7}{$17500$}}%
\psfrag{Anchorbased}[cc][cc]{\scalebox{0.75}{Anchor-based}}%
\psfrag{Algorithm1}[cc][cc]{\scalebox{0.75}{\; Algorithm~\ref{alg:threshold}}}%
\psfrag{Algorithm2}[cc][cc]{\scalebox{0.75}{\; Algorithm~\ref{alg:cost_dependent}}}%
\psfrag{Oracle}[cc][cc]{\scalebox{0.75}{Oracle}}%
\psfrag{y}[bc][bc]{\scalebox{0.8}{MAE}}%
\psfrag{x}[tc][cb]{\scalebox{0.8}{Number of samples $n$}}%
\includegraphics[width=0.5\linewidth,trim=30 0 0 0, clip]{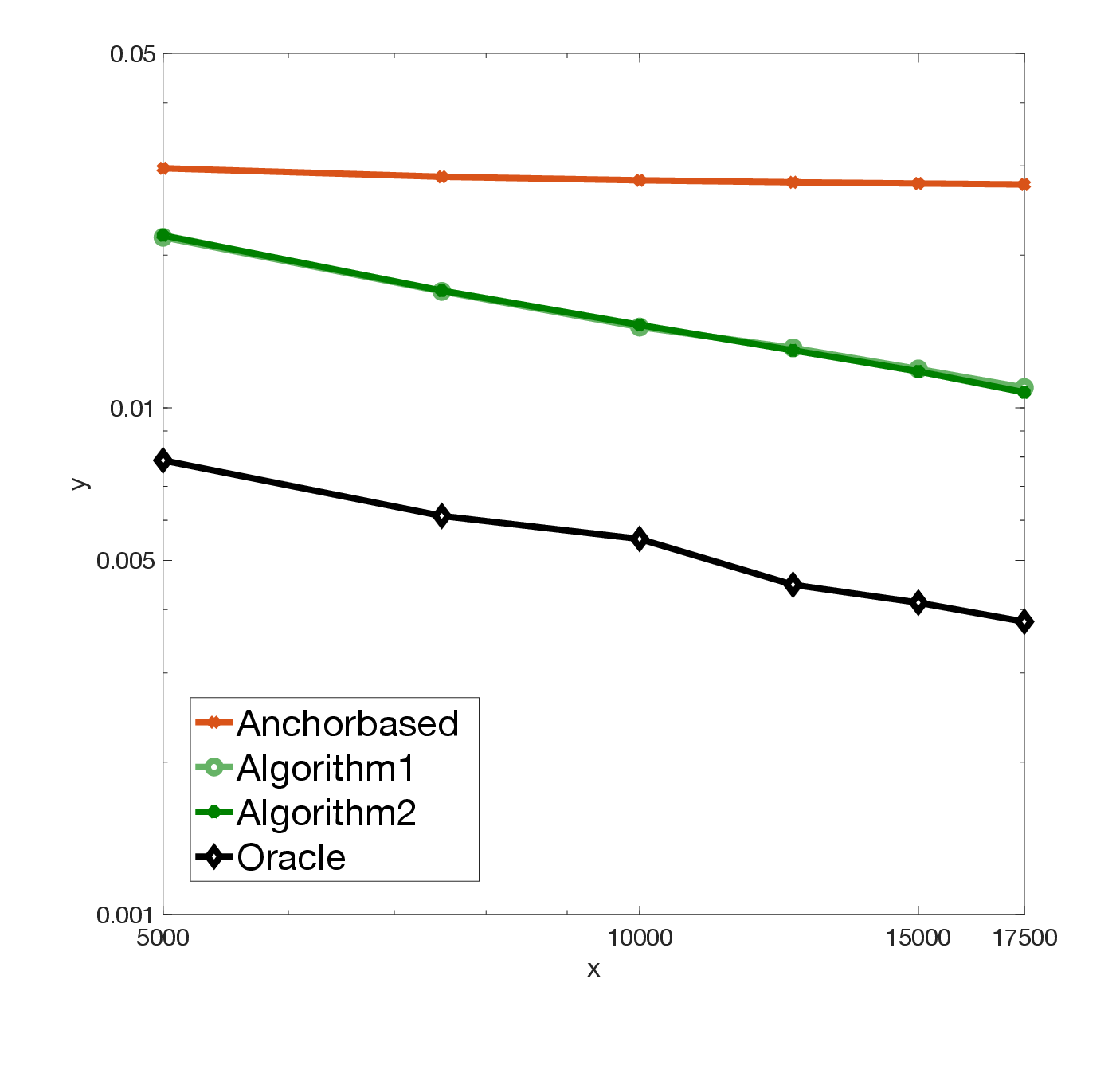}%
}%
\hfill
\subfigure[MNIST dataset]{%
\label{fig:MNIST}%
\psfrag{0.05}[cc][cc]{\scalebox{0.7}{$0.05$\;}}%
\psfrag{0.04}[cc][cc]{\scalebox{0.7}{$0.04$\;}}%
\psfrag{0.03}[cc][cc]{\scalebox{0.7}{$0.03$\;}}%
\psfrag{0.02}[cc][cc]{\scalebox{0.7}{$0.02$\;}}%
\psfrag{0.01}[cc][cc]{\scalebox{0.7}{$0.01$\;}}%
\psfrag{0.005}[cc][cc]{\scalebox{0.7}{$0.005$\;}}%
\psfrag{0.001}[cc][tc]{\scalebox{0.7}{$0.001$\; }}%
\psfrag{10}[cc][rc]{\scalebox{0.48}{$10$}}%
\psfrag{1}[cc][rc]{\scalebox{0.4}{1}}%
\psfrag{2}[cc][rc]{\scalebox{0.4}{2}}%
\psfrag{3}[cc][rc]{\scalebox{0.4}{3}}%
\psfrag{4}[cc][rc]{\scalebox{0.4}{4}}%
\psfrag{Anchorbased}[cc][cc]{\scalebox{0.75}{Anchor-based}}%
\psfrag{Algorithm11}[cc][cc]{\scalebox{0.75}{Algorithm~\ref{alg:threshold}}}%
\psfrag{Algorithm22}[cc][cc]{\scalebox{0.75}{Algorithm~\ref{alg:cost_dependent}}}%
\psfrag{Oracle}[cc][cc]{\scalebox{0.75}{Oracle}}%
\psfrag{y}[bc][bc]{\scalebox{0.8}{MAE}}%
\psfrag{x}[tc][cb]{\scalebox{0.8}{Number of samples $n$}}%
\psfrag{5000}[cc][cc]{\scalebox{0.7}{$5000$}}%
\psfrag{10000}[cc][cc]{\scalebox{0.7}{$10000$}}%
\psfrag{20000}[cc][cc]{\scalebox{0.7}{$20000$}}%
\psfrag{30000}[cc][cc]{\scalebox{0.7}{$30000$}}%
\psfrag{50000}[cc][cc]{\scalebox{0.7}{$50000$}}%
\psfrag{40000}[cc][cc]{\scalebox{0.7}{$ $}}%
\includegraphics[width=0.5\linewidth,trim=10 0 20 0, clip]{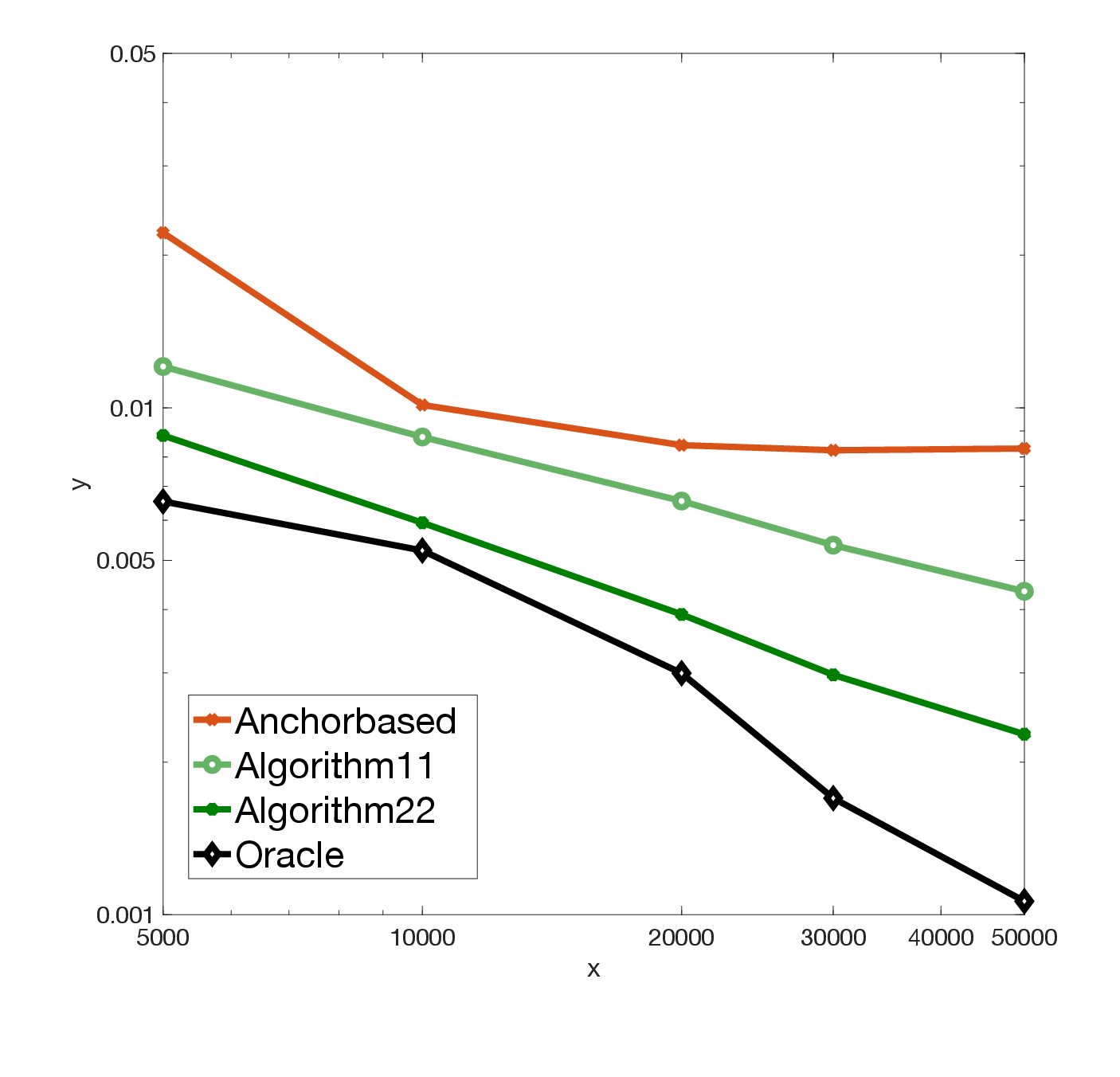}%
}%
\caption{Mean Absolute Error (MAE) of different estimators as the number of training samples $n$ increases. The errors of the presented algorithms decrease with the number of samples in a similar trend to the error of the oracle, whereas the error of the anchor-based method stagnates.}
\label{fig:num_result}
\end{figure*}

Figure~\ref{fig:num_result} shows how the error of different algorithms decreases with the number of samples $n$. 
We implement \cref{alg:threshold,alg:cost_dependent}, and the anchor-based method from \cite{patrini17} with random forests for the Letter dataset (\cref{fig:letter}) and neural networks for the MNIST dataset (\cref{fig:MNIST}). 
Both \cref{alg:threshold,alg:cost_dependent} are executed with $N_+ = n/200$ and $m_1=m_2=n/2$. 
Moreover, we evaluate the algorithms against a strong oracle that learns $\T$ via maximum likelihood estimation (MLE). 
To act as a strong benchmark reference, the oracle has access to feature representations from a clean-data pre-trained network and initializes the MLE optimization directly at the ground-truth $\T$. 
Therefore, this oracle has a very small error corresponding to the finite-sample estimation error due to finite $n$.
Further details regarding the implementation and additional experimental results are provided in Appendix~\ref{app:results}.

\vspace{0.1em}

In \cref{fig:num_result} we show that the proposed algorithms achieve significantly lower error rates than existing methods that rely on pointwise estimates of the class-posteriors, in accordance with our derived performance bounds.
Moreover, the figure shows that the errors of both the proposed algorithms and the oracle exhibit the same decreasing trend as $n$ increases, whereas the error of the anchor-based method plateaus.

\section{Conclusion}
The paper presents a new methodology for the estimation of the label-noise transition matrix based on one-sided selective classification.
In the proposed approach, the transition matrix is estimated by learning selection functions that minimize the false discovery rate subject to a minimum coverage constraint.
We provide finite-sample error bounds for the proposed methodology and show that our approach avoids the curse of dimensionality, unlike existing estimators.
Moreover, we propose effective algorithms that can be implemented with any method for binary classification, and provide their corresponding performance guarantees.
The presented methodology can help to reduce the impact of label noise on supervised learning methods, making model deployment more cost-efficient by allowing the use of data labeled through inexpensive sources.

\paragraph{Limitations:}
The proposed methodology is broadly effective but does not scale well with the number of classes, as it requires solving different binary classification problems for each class to estimate the transition matrix.
Additionally, like other estimators, the approach may struggle under class imbalance, which can affect the precision of the selection functions for minority classes.
These drawbacks are of limited relevance in practice because the sub-problems are naturally parallelizable and can be solved using standard off-the-shelf methods. 
Furthermore, as noisy labels are inexpensive to acquire, the assumption of a sufficiently large training sample is typically satisfied in real-world applications, ensuring the approach remains both computationally and statistically viable.

\section*{Acknowledgments}
Xabier de Juan and Santiago Mazuelas were supported by project PID2022-137063NB-I00 funded by MCIN/AEI/10.13039/501100011033 and the European Union “NextGenerationEU”/PRTR, BCAM Severo Ochoa accreditation CEX2021-001142-S/MICIN/AEI/10.13039/501100011033 funded by the Ministry of Science and Innovation (Spain), and programs BERC-2022--2025 and ELKARTEK funded by the Basque Government.
Yilun Zhu and Clayton Scott were supported in part by the National Science Foundation under award 2551503, and by the Department of Defense, Defense Threat Reduction Agency under award HDTRA1-20-2-0002.
Xabier de Juan acknowledges a predoctoral grant from the Basque Government.

\bibliography{refs}
\bibliographystyle{unsrt}

%% file: apendice.tex
\appendix

\section{Proofs}
\label{app:proofs}
The proofs shown below utilize the following lemma that provides general upper bounds for the estimation error.

\begin{lemma}
\label{lemma:main}
Fix $i,j\in\set{Y}$ and let $\set{F}$ be a family of measurable subsets of $\set{X}$.
Then, for every $A\in\set{F}$
    \begin{align*}
    |T_{i,j}-\widehat T_{i,j} | &\leq C_T \left( R_\set{F}^{(j)} + \sup_{A'\in \set{F}} \P(\widetilde Y = j \mid X \in A')  - \P(\widetilde Y = j \mid X \in A) \right) \\
    & \quad + |\P(\widetilde Y = i \mid X \in A)- \widehat  T_{i,j}|
    \end{align*}
    where $C_T = 2/(T_{j,j}-\max\limits_{l\neq j} T_{j,l})$ and $R_{\set{F}}^{(j)} = \inf\limits_{A'\in \set{F}} \P(Y \neq j \mid X \in A')$.
\end{lemma}
\begin{proof}
By the triangle inequality we get
\begin{align*}
    |T_{i,j}-\widehat T_{i,j} | \leq |T_{i,j} -\P(\widetilde Y = i \mid X \in A)|+ |\P(\widetilde Y = i \mid X \in A)- \widehat T_{i,j}| .
\end{align*}

We apply Hölder's inequality in the first term so that
\begin{align*}
    |T_{i,j} -\P(\widetilde Y = i \mid X \in A)| &= \left| T_{i,j} - \sum_{l=1}^{|\set{Y}|} T_{i,l} \P(Y = l \mid X \in A)\right|\\
    &= |(T_{i,1}, T_{i,2},\ldots, T_{i,|\set{Y}|})^{\up{T}}(e_j - \P(\bm{Y}  \mid X \in A))| \\
    &\leq  \norm{e_j - \P( \bm{Y} \mid X \in A)}_1 = 2(1 - \P(Y= j \mid X \in A))
\end{align*}
where $e_j$ denotes the $j$-th vector from the canonical basis and $ \P(\bm{Y}\mid X \in A)$ is the vector whose $t$-th component is $ \P(Y = t \mid X \in A)$.

Using the definition of $C_T$, we have
\begin{align}
    1 - \P(Y= j \mid X \in A) \leq \frac{C_T}{2}({T_{j,j}-\P(\widetilde Y = j \mid X \in A)}).
    \label{eq:ineq1}
\end{align}
Such a bound can be obtained as follows.
\begin{align}
        \P(\widetilde Y = j \mid X \in A ) &
        = \sum_{l=1}^{|\set{Y}|} T_{j,l} \P( Y = l \mid X \in A)  \nonumber \\
        &= T_{j,j} \P(Y=j \mid X \in A) + \sum_{l=1,l\neq j}^{|\set{Y}|} T_{j,l} \P( Y = l \mid X \in A) \label{eq2}. 
\end{align}
Therefore,
    \begin{align*}
        \P(Y=j \mid X \in A) &= \frac{\P(\widetilde Y = j \mid X \in A ) - \sum_{l=1,l\neq j}^{|\set{Y}|}  T_{j,l} \P(Y=l\mid X \in A)}{T_{j,j}}\\
        &\geq \frac{\P(\widetilde Y = j \mid X \in A ) - \max_{l\neq j} \{T_{j,l} \} ( 1-  \P(Y=j\mid X \in A))}{T_{j,j}}.
    \end{align*}
    Thus, 
    \begin{align*}
        \left(1 - \frac{\max_{l\neq j} T_{j,l} }{T_{j,j}} \right) \P(Y=j \mid X \in A)  \geq \frac{\P(\widetilde Y = j \mid X \in A )-\max_{l\neq j} \{T_{j,l} \}}{T_{j,j}}
    \end{align*}
    whence \eqref{eq:ineq1} follows since ${\max\limits_{l\neq j}\{T_{j,l} \} }/{T_{j,j}} < 1$.

The right hand side in \eqref{eq:ineq1} can be written as
\begin{align*}
    T_{j,j}-\P(\widetilde Y = j \mid X \in A) &=  T_{j,j}  - \sup_{A'\in\set{F}}  \P(\widetilde Y = j \mid X \in A') \\
    & \quad + \sup_{A'\in\set{F}}  \P(\widetilde Y = j \mid X \in A') -\P(\widetilde Y = j \mid X \in A).
\end{align*}
Then, the result is obtained because
\begin{align*}
    T_{j,j} - \sup_{A'\in\set{F}}  \P(\widetilde Y = j \mid X \in A') \leq 1-\sup_{A'\in\set{F}}  \P( Y = j \mid X \in A') = R_{\set{F}}^{(j)}
\end{align*}
since
\begin{align*}
    \sup_{A'\in\set{F}}  \P(\widetilde Y = j \mid X \in A') \geq T_{j,j} \sup_{A'\in\set{F}}  \P( Y = j \mid X \in A')
\end{align*}
as a consequence of \eqref{eq2}, and we also have
\begin{align*}
    T_{j,j} \sup_{A'\in\set{F}}  \P( Y = j \mid X \in A') \geq T_{j,j} +  \sup_{A'\in\set{F}}  \P( Y = j \mid X \in A') -1
\end{align*}
since $(T_{j,j}-1)\big( \sup\limits_{A'\in\set{F}}  \P( Y = j \mid X \in A')  -1\big) \geq 0$.
\end{proof}

\subsection{Proof of Theorem~\ref{thm:performance_guarantees}}
\label{proof:thm:performance_guarantees}

    The result is a particular case of Lemma~\ref{lemma:main} taking
    \begin{align*}
         \set{F}= \{A_h \subseteq \set{X}: h \colon \set{X} \to \{\accept,\reject\} \st \P(h(X) = \accept) \geq \gamma \} 
    \end{align*}
    where $A_h = \{ x \in \set{X} : h(x) = \accept \}$ is the set of accepted instances of a selection function \mbox{$h \colon \set{X} \to \{\accept,\reject\}$}, and $A = \{ x \in\set{X} : h^{(j)}(x) = \accept  \}$.
    Therefore, Lemma~\ref{lemma:main} implies
    \begin{align*}
        |T_{i,j}-\widehat T_{i,j} | \leq C_T\left( R_{\gamma}^{(j)} + \epsilon_{\opt}(h^{(j)}) \right) + |\P(\widetilde Y = i \mid X \in A)-\widehat  T_{i,j}|        
    \end{align*}
    since $R_{\gamma}^{(j)}  = R_\set{F}^{(j)}$ and $\epsilon_{\opt}(h^{(j)}) = \sup\limits_{A_h\in \set{F}} \P(\widetilde Y = j \mid X \in A_h)  - \P(\widetilde Y = j \mid X \in A)$.

    Finally, the desired result follows since for each $i\in\set{Y}$
    \begin{align}
    |\P(\widetilde Y = i \mid h^{(j)}(X) = \accept) - \widehat  T_{i,j}(h^{(j)})| \leq \sqrt{\frac{2\log(1/\delta)}{\#_m(h^{(j)})}} \label{eq:chernoff}
    \end{align}
    with probability at least $1-\delta$ (over the draw of the samples in $S$).
    To obtain \eqref{eq:chernoff}, note that conditioned on the event that the number of accepted samples is $k=\#_m(h^{(j)})$, the random variable $k \widehat T_{i,j}(h^{(j)})$ follows a binomial distribution with parameters $k$ and probability \mbox{$\P(\widetilde Y = i \mid h^{(j)}(X) = \accept)$}.
    Then, the bound in \eqref{eq:chernoff} follows by the Chernoff bound conditionally on that event, and using the law of total probability varying the values of $k$.

    Finally, taking a union bound over the $|\set{Y}|$ classes we obtain
    \begin{align*}
    \max_{i\in\set{Y}}|\P(\widetilde Y = i \mid h^{(j)}(X) = \accept) - \widehat  T_{i,j}(h^{(j)})| \leq \sqrt{\frac{2\log(|\set{Y}|/\delta)}{\#_m(h^{(j)})}} 
    \end{align*}
    that holds with probability at least $1-\delta$. \qed

\subsection{Proof of Proposition~\ref{prop:anchors}}
\label{proof:prop:anchors}
    The result is proven using Lemma~\ref{lemma:main} with $\set{F}=\{ \{x\} : x\in\set{X}\}$ and $A = \{ x^{(j)}\}$.
    Therefore,
    \begin{align*}
        |T_{i,j}-\widehat T_{i,j}^{\anchor}(  x^{(j)} ) | \leq C_T\left( R_{0}^{(j)} + \epsilon_{\opt}^{\anchor}( x^{(j)})  \right) + |{{\eta}}^{\, \noisy}_i(x^{(j)}) -\widehat{{\eta}}^{\, \noisy}_i(x^{(j)})|        
    \end{align*}
    since $R_{0}^{(j)}  = R_\set{F}^{(j)}$ and $\epsilon_{\opt}^{\anchor}( x^{(j)})  = \sup\limits_{A'\in \set{F}} \P(\widetilde Y = j \mid X \in A')  - \P(\widetilde Y = j \mid X \in A)$. \qed

\subsection{Proof of Theorem~\ref{thm:threshold}}
\label{proof:thm:threshold}
    The result is proven using Lemma~\ref{lemma:main} with
    \begin{align*}
         \set{F}= \{A_h \subseteq \set{X}: h \colon \set{X} \to \{\accept,\reject\} \st \P(h(X) = \accept) \geq \gamma \} 
    \end{align*}
    where $A_h = \{ x \in \set{X} : h(x) = \accept \}$ is the set of accepted instances of a selection function \mbox{$h \colon \set{X} \to \{\accept,\reject\}$}, and $A = \{ x \in\set{X} : h^{(j)}_{\widehat \tau}(x) = \accept  \}$.
    Moreover, we have $A\in\set{F}$ since $\gamma = \P(h_{\widehat \tau}^{(j)}(X)=\accept)$ by hypothesis.
    Therefore, Lemma~\ref{lemma:main} implies
    \begin{align*}
        |T_{i,j}-\widehat T_{i,j} | \leq C_T\left( R_{\gamma}^{(j)} + \epsilon_{\opt}(h^{(j)}_{\widehat \tau}) \right) + |\P(\widetilde Y = i \mid X \in A)-\widehat  T_{i,j}|        
    \end{align*}
    since $R_{\gamma}^{(j)}  = R_\set{F}^{(j)}$ and $\epsilon_{\opt}(h^{(j)}_{\widehat \tau}) = \sup\limits_{A_h\in \set{F}} \P(\widetilde Y = j \mid X \in A_h)  - \P(\widetilde Y = j \mid X \in A)$.
    Moreover, since $\widehat T_{i,j}$ is the empirical version of $\P(\widetilde Y = i \mid h^{(j)}_{\widehat \tau}(X) = \accept)$, the uniform convergence of empirical conditional measures (Theorem~8 in \cite{balsubramani19}) implies
    \begin{align*}
    |\P(\widetilde Y = i \mid h^{(j)}_{\widehat \tau}(X) = \accept) - \widehat  T_{i,j}| \leq \sqrt{\frac{1000(\log(8 m_2)+\log(4/\delta))}{\#_{m_2}(h^{(j)}_{\widehat\tau})}}
    \end{align*}
    with probability at least $1-\delta$ (over the draw of the samples indexed by $S_2$).
    Taking a union bound over the $|\set{Y}|$ classes we obtain
    \begin{align*}
    \max_{i\in\set{Y}}|\P(\widetilde Y = i \mid h^{(j)}(X) = \accept) - \widehat  T_{i,j}(h^{(j)})| \leq \sqrt{\frac{1000(\log(8 m_2)+\log(4|\set{Y}|/\delta))}{\#_{m_2}(h^{(j)}_{\widehat\tau})}}
    \end{align*}
    with probability at least $1-\delta$ (over the draw of the samples indexed by $S_2$).

    Then, the result is obtained by bounding the suboptimality gap
    \begin{align*}
        \epsilon_{\opt}(h_{\widehat \tau}^{(j)}) =  \P(\widetilde Y \neq j \mid h^{(j)}_{\widehat \tau}(X) = \accept) - \min_h \{\P(\widetilde Y \neq j \mid h(X) = \accept) \st \P(h(X)=\accept)\geq\gamma\}.
    \end{align*}
    Note that $\widetilde R_\gamma^{(j)}$ is achieved by thresholding the noisy class-posterior at $\widehat\tau_\eta$.
    Therefore,
    \begin{align*}
    \epsilon_{\opt}(h_{\widehat \tau}^{(j)}) &= \P(\widetilde Y \neq j \mid s^{(j)}(X) \geq \widehat\tau) -\P(\widetilde Y \neq j \mid \eta_j^{\, \noisy}(X) \geq \widehat\tau_\eta)
    \\ &= \frac{1}{\gamma}( \P(\widetilde Y \neq j, s^{(j)}(X) \geq \widehat\tau) -  \P(\widetilde Y \neq j, \eta_j^{\, \noisy}(X) \geq \widehat\tau_\eta) ).
    \end{align*}
    where the last equality holds since $ \P(h_{\widehat \tau}^{(j)}(X) = \accept) = \gamma = \P(\eta_j^{\, \noisy}(X) \geq\widehat\tau_\eta)$.
    Finally, it is easy to check that
    \begin{align*}
        \P(\widetilde Y \neq j, s^{(j)}(X) \geq \widehat\tau) -  \P(\widetilde Y \neq j, \eta_j^{\, \noisy}(X) \geq \widehat\tau_\eta) &=  \P(\widetilde Y \neq j,  s^{(j)}(X) \geq \widehat\tau,  \eta_j^{\, \noisy}(X) < \widehat\tau_\eta) \\ 
        & \quad -   \P(\widetilde Y \neq j, s^{(j)}(X) <\widehat\tau,  \eta_j^{\, \noisy}(X) \geq \widehat\tau_\eta)\\
        & = \mathcal{E}_{\text{rank}}.
    \end{align*}
    \qed

\subsection{Proof of Theorem~\ref{thm:cost_dependent}}
\label{proof:thm:cost_dependent}
    The result is proven using Lemma~\ref{lemma:main}.
    Specifically, let
    \begin{align*}
         \set{F}= \{A_h \subseteq \set{X}: h \colon \set{X} \to \{\accept,\reject\} \st \P(h(X) = \accept) \geq \gamma \} 
    \end{align*}
    where $A_h = \{ x \in \set{X} : h(x) = \accept \}$ is the set of accepted instances of a selection function \mbox{$h \colon \set{X} \to \{\accept,\reject\}$}, and let $A = \{ x \in\set{X} : h^{(j)}_{\widehat c}(x) = \accept  \}$.
    Moreover, we have $A\in\set{F}$ since $\gamma = \P(h^{(j)}_{\widehat c}(X)=\accept)$ by hypothesis.
    Therefore, Lemma~\ref{lemma:main} implies
    \begin{align*}
        |T_{i,j}-\widehat T_{i,j} | \leq C_T\left( R_{\gamma}^{(j)} + \epsilon_{\opt}(h^{(j)}_{\widehat c}) \right) + |\P(\widetilde Y = i \mid X \in A)-\widehat  T_{i,j}|        
    \end{align*}
    since $R_{\gamma}^{(j)}  = R_\set{F}^{(j)}$ and $\epsilon_{\opt}(h^{(j)}_{\widehat c}) = \sup\limits_{A_h\in \set{F}} \P(\widetilde Y = j \mid X \in A_h)  - \P(\widetilde Y = j \mid X \in A)$.

    Moreover, by applying the Chernoff bound conditionally, for every $c$ in the grid
    \begin{align}
    \max_{i\in\set{Y}}|\P(\widetilde Y = i \mid h^{(j)}_c(X) = \accept) - \widehat  T_{i,j}(h^{(j)}_c;S_2)| \leq \sqrt{\frac{2\log(|\set{Y}|/\delta)}{\#_{m_2}(h^{(j)}_c)}} 
    \label{eq:chernoff2}
    \end{align}
    holds with probability at least $1-\delta$.
    To obtain \eqref{eq:chernoff2}, note that conditioned on the event that the number of accepted samples is $k=\#_{m_2}(h^{(j)}_c)$, the random variable $k \widehat T_{i,j}(h^{(j)}_c)$ follows a binomial distribution with parameters $k$ and probability $\P(\widetilde Y = i \mid h^{(j)}_c(X) = \accept)$.
    Then, the bound in \eqref{eq:chernoff2} follows by the Chernoff bound conditionally on that event, and using the law of total probability varying the values of $k$.

    Therefore, taking a union bound over the grid of size $\floor{1/\epsilon}$, we obtain for $c=\widehat c$
    \begin{align*}
    \max_{i\in\set{Y}} |\P(\widetilde Y = i \mid h^{(j)}_{\widehat c}(X) = \accept) - \widehat  T_{i,j}| \leq \sqrt{\frac{2\log(|\set{Y}|/(\delta\epsilon))}{\#_{m_2}( h_{\widehat c}^{(j)})}}
    \end{align*}
    with probability at least $1-\delta$ (over the draw of the samples indexed by $S_2$). 

    Then, the result is obtained by bounding the suboptimality gap
    \begin{align*}
        \epsilon_{\opt}(h_{\widehat c}^{(j)}) =  \P(\widetilde Y \neq j \mid h^{(j)}_{\widehat c}(X) = \accept) - \min_h \{\P(\widetilde Y \neq j \mid h(X) = \accept) \st \P(h(X)=\accept)\geq\gamma\}.
    \end{align*}
    Let $h_*^{(j)}$ be a solution of the one-sided selective classification problem in \eqref{eq:selective_classification} such that \mbox{$\P(h^{(j)}_*(X) = \accept)=\gamma$}.
    Then, 
    \begin{align*}
    \epsilon_{\opt}(h_{\widehat c}^{(j)}) &= \P(\widetilde Y \neq j \mid h^{(j)}_{\widehat c}(X) = \accept) -\P(\widetilde Y \neq j \mid h^{(j)}_*(X) = \accept)
    \\ &= \frac{1}{\gamma}( \P(\widetilde Y \neq j, h^{(j)}_{\widehat c}(X) = \accept) -  \P(\widetilde Y \neq j, h^{(j)}_*(X) = \accept) ).
    \end{align*}
    where the last equality holds since $\P(h_{\widehat c}^{(j)}(X) = \accept) = \gamma = \P(h^{(j)}_*(X) = \accept)$.
    Finally, a simple calculation yields
    \begin{align*}
         \P(\widetilde Y \neq j, & h^{(j)}_{\widehat c}(X) = \accept) -  \P(\widetilde Y \neq j, h^{(j)}_*(X) = \accept) \\ 
        & =\P(\widetilde Y \neq j, h^{(j)}_{\widehat c}(X) = \accept) -  \P(\widetilde Y \neq j, h^{(j)}_*(X) = \accept) \\
        & \quad  + \widehat c(\P(h^{(j)}_{\widehat c}(X) = \reject) - \P(h^{(j)}_*(X) = \reject)  ) \\
        & \quad - \widehat c(\P(h^{(j)}_{\widehat c}(X) = \reject) - \P(h^{(j)}_*(X) = \reject)  )\\
        &=\E[\ell_{\widehat c}^{(j)}(h^{(j)}_{\widehat c}(X),\widetilde Y)] - \E[\ell_{\widehat c}^{(j)}(h^{(j)}_*(X),\widetilde Y)]\\
        & \quad +\widehat c(\P(h^{(j)}_{\widehat c}(X) = \accept) - \P(h^{(j)}_*(X) = \accept)  )\\
        & \leq \E[\ell_{\widehat c}^{(j)}(h^{(j)}_{\widehat c}(X),\widetilde Y)] - \inf_h\E[\ell_{\widehat c}^{(j)}(h(X),\widetilde Y)]\\
        & = \mathcal{E}_{\text{cost}}
    \end{align*}
    where in the last inequality we have used that $\P(h_{\widehat c}^{(j)}(X) = \accept) = \gamma = \P(h^{(j)}_*(X) = \accept)$ and $\inf_h\E[\ell_{\widehat c}^{(j)}(h(X),\widetilde Y)]\leq \E[\ell_{\widehat c}^{(j)}(h^{(j)}_*(X),\widetilde Y)]$. \qed

\subsection{Proof of Theorem~\ref{thm:cost_dependent_surrogate}}
\label{proof:thm:cost_dependent_surrogate}
The result is a consequence of Theorem~\ref{thm:cost_dependent}.
Then, the result is obtained by establishing the bound
\begin{equation}
\label{eq:excess_risk}
\begin{split}
2\psi_\Phi\left( \frac{\E[\ell_{\widehat c}^{(j)}(h^{(j)}_{\widehat c}(X),\widetilde Y)] - \inf_h\E[\ell_{\widehat c}^{(j)}(h(X),\widetilde Y)]}{2}\right)&\leq \E[L_{\widehat c}^{(j)}( f_{\widehat c}^{(j)}(X),\widetilde Y)] \\ 
        & \quad - \inf_{f}\E[L_{\widehat c}^{(j)}( f(X),\widetilde Y)]
\end{split}
\end{equation}
    which we prove in what follows.
    
    Let $x\in\set{X}$ be arbitrary and for notational convenience let $\eta(x) = \P(\widetilde Y = j \mid X = x)$.
    We define the conditional $L_{\widehat c}^{(j)}$-risk as
    \begin{align*}
        C_{\Phi}(x;\alpha) \coloneqq \E[ L_{\widehat c}^{(j)}(\alpha,\widetilde Y) \mid X = x], \quad \alpha\in\R.
    \end{align*}
    A simple calculation produces
    \begin{equation}
    \label{eq:C_Phi}
        \begin{split}
            C_{\Phi}(x;\alpha) &= (1-\eta(x))\Phi(-\alpha) + \widehat c\Phi(\alpha)\\
        &= S(x)\left( (1-\eta'(x)) \Phi(-\alpha) + \eta'(x)\Phi(\alpha)    \right)
        \end{split}
    \end{equation}
    where $S(x) = 1-\eta(x)+\widehat c$ and $\eta'(x) = \widehat c/S(x)\in[0,1]$.
    Now define the conditional $\ell_{\widehat c}^{(j)}$-risk as
    \begin{align*}
        C(x;a) \coloneqq  \E[ \ell_{\widehat c}^{(j)}(a,\widetilde Y) \mid X = x], \quad a\in\{-1,+1\}
    \end{align*}
    where for notational convenience $\ell_{\widehat c}^{(j)}(a,\widetilde Y)$ is well defined for $a\in\{-1,1\}$ (we identify $\accept$ with $+1$ and $\reject$ with $-1$).
    By \cite[Sec.~2.2]{bartlett06} we have
    \begin{align}
        \inf_{f}\E[L_{\widehat c}^{(j)}( f(X),\widetilde Y)] = \E_X \left[\inf_{\alpha\in\R} C_{\Phi}(X;\alpha)\right], \label{eq:cond_risk1}\\
        \inf_h\E[\ell_{\widehat c}^{(j)}(h(X),\widetilde Y)] = \E_X \left[\inf_{a\in\{-1,+1\}} C(X;a)\right]. \label{eq:cond_risk2}
    \end{align}
    Define the excess conditional risks
    \begin{align*}
        \Delta C_\Phi(x; \alpha) \coloneqq C_{\Phi}(x;\alpha) - \inf_{\alpha\in\R} C_{\Phi}(x;\alpha),\\
        \Delta C(x; a) \coloneqq C(x;a) - \inf_{a\in\{-1,+1\}} C(x;a).
    \end{align*}
    A simple calculation yields
    \begin{align}
        \Delta C( x;a) = |\eta(x) - (1-{\widehat c})| \I\{ a \neq \sign(\eta(x) -(1-{\widehat c}))\}.
        \label{eq:property1}
    \end{align}

    Now we proceed with the proof of \eqref{eq:excess_risk} for an arbitrary score function $f^{(j)}$ and its corresponding binary classifier $h^{(j)}$, i.e., $h^{(j)}(x) = \sign(f^{(j)}(x))$.

    Case 1: $\sign(2\eta'(x)-1) = \sign(f^{(j)}(x))$.
    We have \mbox{$\sign(2\eta'(x)-1) = \sign(\eta(x)-(1-\widehat c))$}, and thus by \eqref{eq:property1} we get
    \begin{align*}
        \Delta C( x;h^{(j)}(x)) = |\eta(x) - (1-\widehat c)| \I\{ \sign(f^{(j)}(x)) \neq \sign(\eta(x) -(1-\widehat c)) \}= 0.
    \end{align*}
    Since the $\psi_\Phi$-transform satisfies $\psi_\Phi(0) = 0$ as shown in \cite[Lemma~2]{bartlett06}, the bound
    \begin{align*}
        \Delta C_\Phi( x;f^{(j)}(x)) \geq 0 = 2  \psi_\Phi\left( \frac{\Delta C(x;h^{(j)}(x))}{2}  \right)
    \end{align*}
    trivially holds.

    Case 2: $\sign(2\eta'(x)-1) \neq \sign(f^{(j)}(x))$.
    For $t\in[0,1]$ we define
    \begin{align*}
        H(t) &= \inf_{\alpha\in\R} \left( (1-t) \Phi(-\alpha) + t\Phi(\alpha)    \right),\\
        H^{-}(t) &= \inf_{\alpha\st \alpha(2t-1)\leq0} \left( (1-t) \Phi(-\alpha) + t\Phi(\alpha)    \right)
    \end{align*}
    following \cite[Def.~1]{bartlett06}.
    Then, since $\sign(2\eta'(x)-1) \neq \sign(f^{(j)}(x))$, we have
    \begin{align*}
        \Delta C_\Phi( x;f^{(j)}(x)) &\geq S(x)(H^{-}(\eta'(x)) - H(\eta'(x))) && \text{(By \eqref{eq:C_Phi})}\\
        & = S(x) \psi_\Phi(2\eta'(x)-1) && \text{(definition of $\psi$, \cite[Def.~2]{bartlett06})}\\
        &=  S(x) \psi_\Phi(|2\eta'(x)-1|) && \text{(symmetry of $\psi$, \cite[Lemma~2]{bartlett06})}\\
        &= S(x) \psi_\Phi\left( \frac{|\eta(x)-(1-\widehat c)|}{S(x)} \right) && \text{($|2\eta'(x)-1| =  \frac{|\eta(x)-(1-\widehat c)|}{S(x)}$)}\\
        &= S(x) \psi_\Phi\left( \frac{\Delta C(x;h^{(j)}(x))}{S(x)} \right) && \text{($|\eta(x)-(1-\widehat c)| =  \Delta C(x;h^{(j)}(x))$)}\\
        & =2 \frac{S(x)}{2} \psi_\Phi\left( \frac{\Delta C(x;h^{(j)}(x))}{S(x)} \right)\\
        & \geq 2 \psi_\Phi\left( \frac{\Delta C(x;h^{(j)}(x))}{S(x)} \frac{S(x)}{2} \right) && \text{($S(x)\leq 2$ and $\psi_\Phi$ is convex)} \\
        & = 2  \psi_\Phi\left( \frac{\Delta C(x;h^{(j)}(x))}{2}  \right).
    \end{align*}

    Therefore, we have shown that for every $x\in \set{X}$,
    \begin{align*}
        \Delta C_\Phi( x;f^{(j)}(x))\geq 2  \psi_\Phi\left( \frac{\Delta C(x;h^{(j)}(x))}{2}  \right).
    \end{align*}
    Finally, since $\psi_\Phi$ is convex, by Jensen's inequality we obtain
    \begin{align*}
        \E_X[\Delta C_\Phi(X;f^{(j)}(X))] \geq \E_X\left[ 2  \psi_\Phi\left( \frac{\Delta C(X;h^{(j)}(X))}{2}  \right) \right] \geq 2 \psi_\Phi\left( \frac{\E_X[\Delta C(X;h^{(j)}(X))]}{2} \right)
    \end{align*}
    which proves \eqref{eq:excess_risk} because
    \begin{align*}
        \E_X[\Delta C_\Phi(X;f^{(j)}(X))]  = \E[L_{\widehat c}^{(j)}( f^{(j)}(X),\widetilde Y)] - \inf_{f}\E[L_{\widehat c}^{(j)}( f(X),\widetilde Y)]
    \end{align*}
    and
    \begin{align*}
        \E_X[\Delta C(X;h^{(j)}(X))] = \E[\ell_{\widehat c}^{(j)}(h^{(j)}(X),\widetilde Y)] - \inf_h\E[\ell_{\widehat c}^{(j)}(h(X),\widetilde Y)]
    \end{align*}
    hold by \eqref{eq:cond_risk1} and \eqref{eq:cond_risk2}, respectively.
    \qed

\section{Analysis of the effect of the number of accepted samples}
\label{app:algorithms}
In this appendix, we provide an analysis on the effect of the hyperparameter $N_+$ for \cref{alg:threshold,alg:cost_dependent}.
In practice, the optimal choice of the lower bound for the number of accepted samples $N_+$ must balance the need for a sufficiently large number of accepted instances (to reduce the variance of the estimator) with the requirement of maintaining a low false discovery rate (to minimize the bias).

For Algorithm~\ref{alg:threshold}, we analyze the impact of $N_+$ by observing the empirical objective function $\widehat T_{j,j}(h^{(j)}_{\tau_k};S_2)$ as a function of the number of accepted samples $k\in\{1,2,\ldots,m_2\}$.

When plotting $(k,\widehat T_{j,j}(h^{(j)}_{\tau_k};S_2))$, we typically observe three distinct regimes.
\begin{enumerate}
    \item When the number of accepted samples is very small, the empirical probability $\widehat T_{j,j}(h^{(j)}_{\tau_k};S_2)$ fluctuates due to high variance.

    \item As the number of accepted samples increases, the variance decreases, and $\widehat T_{j,j}(h^{(j)}_{\tau_k};S_2)$ stabilizes at a constant value which is close to the true transition probability $T_{j,j}$.
    Here, the corresponding selection functions accept enough instances to provide a low-variance estimate without accepting samples belonging to other true classes $y\neq j$, thereby keeping the false discovery rate small.
    
    \item If the number of accepted samples increases beyond the plateau, the corresponding selection functions are forced to accept samples with labels $y\neq j$. 
    As a consequence, the estimate $\widehat T_{j,j}(h^{(j)}_{\tau_k};S_2)$ drops sharply from $T_{j,j}$.
\end{enumerate}

Choosing the maximum number of accepted samples inside the plateau guarantees the lowest possible variance for the estimator while ensuring that the bias remains controlled. 

We provide experimental results on the CIFAR-10 dataset under different noise scenarios that illustrate the three regimes (see \cref{fig:high_noise}), and we provide further implementation details in Appendix~\ref{app:results}.

\begin{figure*}[htb!]
\centering
\subfigure[CIFAR-10 under uniform noise with $T_{j,j}=0.5$.]{

\label{fig:n_plus_thr_flip_45}

\psfrag{1}[cc][rc]{\scalebox{0.5}{1}}
\psfrag{0.9}[cc][cc]{\scalebox{0.5}{0.9}}
\psfrag{0.8}[cc][cc]{\scalebox{0.5}{0.8}}
\psfrag{0.7}[cc][cc]{\scalebox{0.5}{0.7}}
\psfrag{0.6}[cc][cc]{\scalebox{0.5}{0.6}}
\psfrag{0.5}[cc][cc]{\scalebox{0.5}{0.5}}
\psfrag{0.4}[cc][cc]{\scalebox{0.5}{0.4}}
\psfrag{0.3}[cc][cc]{\scalebox{0.5}{0.3}}
\psfrag{0.2}[cc][cc]{\scalebox{0.5}{0.2}}
\psfrag{0.1}[cc][cc]{\scalebox{0.5}{0.1}}
\psfrag{0}[cc][rc]{\scalebox{0.5}{0}}

\psfrag{10}[cc][rc]{\scalebox{0.6}{$10$}}
\psfrag{1}[cc][rc]{\scalebox{0.5}{\, 1}}
\psfrag{2}[cc][rc]{\scalebox{0.5}{\, 2}}
\psfrag{3}[cc][rc]{\scalebox{0.5}{\, 3}}
\psfrag{4}[cc][rc]{\scalebox{0.5}{\, 4}}

\psfrag{y}[bc][bc]{\scalebox{0.7}{$\widehat T_{j,j}(h^{(j)}_{\tau_k};S_2)$}}
\psfrag{x}[cc][cc]{\scalebox{0.7}{Number of accepted samples $k$}}

\includegraphics[width=0.48\linewidth]{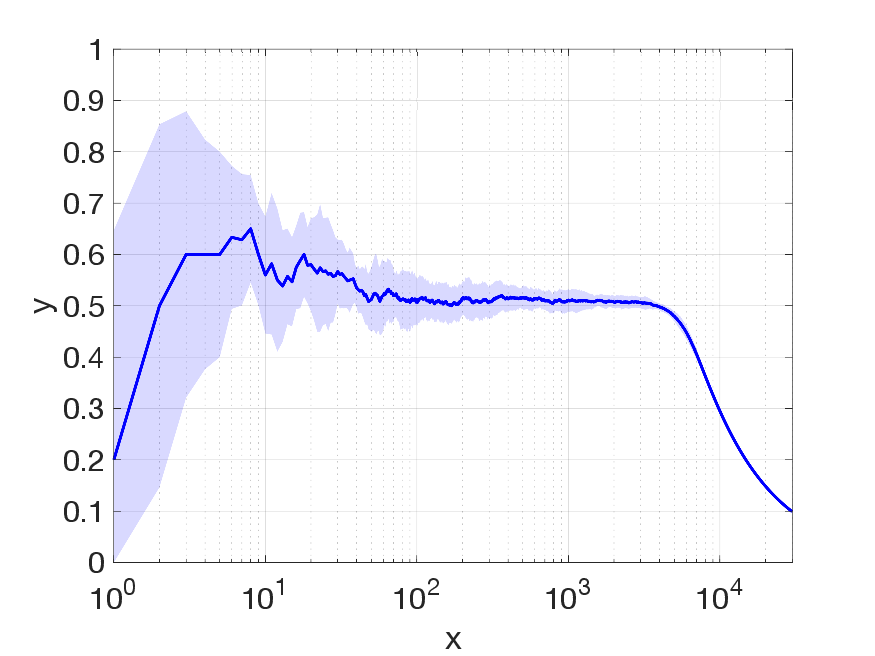}

}%
\subfigure[CIFAR-10 under flip noise with $T_{j,j}=0.55$.]{

\label{fig:n_plus_thr_uniform_50}

\psfrag{1}[cc][rc]{\scalebox{0.5}{1}}
\psfrag{0.9}[cc][cc]{\scalebox{0.5}{0.9}}
\psfrag{0.8}[cc][cc]{\scalebox{0.5}{0.8}}
\psfrag{0.7}[cc][cc]{\scalebox{0.5}{0.7}}
\psfrag{0.6}[cc][cc]{\scalebox{0.5}{0.6}}
\psfrag{0.5}[cc][cc]{\scalebox{0.5}{0.5}}
\psfrag{0.4}[cc][cc]{\scalebox{0.5}{0.4}}
\psfrag{0.3}[cc][cc]{\scalebox{0.5}{0.3}}
\psfrag{0.2}[cc][cc]{\scalebox{0.5}{0.2}}
\psfrag{0.1}[cc][cc]{\scalebox{0.5}{0.1}}
\psfrag{0}[cc][rc]{\scalebox{0.5}{0}}

\psfrag{10}[cc][rc]{\scalebox{0.6}{$10$}}
\psfrag{1}[cc][rc]{\scalebox{0.5}{\, 1}}
\psfrag{2}[cc][rc]{\scalebox{0.5}{\, 2}}
\psfrag{3}[cc][rc]{\scalebox{0.5}{\, 3}}
\psfrag{4}[cc][rc]{\scalebox{0.5}{\, 4}}

\psfrag{y}[bc][bc]{\scalebox{0.7}{$\widehat T_{j,j}(h^{(j)}_{\tau_k};S_2)$}}
\psfrag{x}[cc][cc]{\scalebox{0.7}{Number of accepted samples $k$}}

  \includegraphics[width=0.48\linewidth]{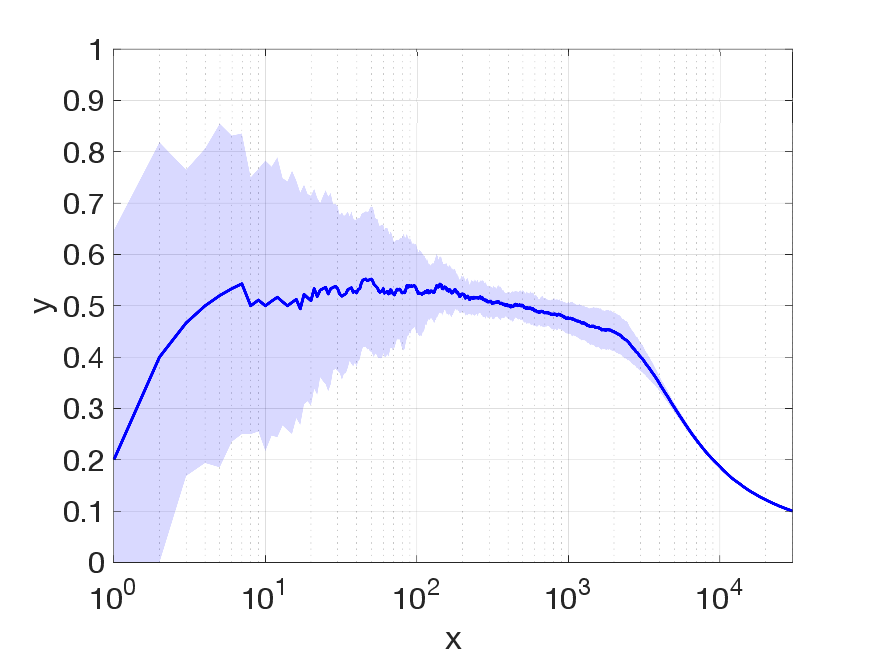}

}%
\caption{The three regimes described above appear on CIFAR-10. Moreover, the value of the plateau is close to $T_{j,j}$.}
\label{fig:high_noise}
\end{figure*}

In Figure~\ref{fig:comparison_N} we show that when the number of instances is small, the size of the plateau is smaller and there is more variability, leading to higher estimation errors.

\begin{figure*}[htb!]
\centering
\subfigure[CIFAR-10 using 10,000 samples.]{

\label{fig:n_plus_thr_uniform_20_10000}

\psfrag{1}[cc][rc]{\scalebox{0.5}{1}}
\psfrag{0.9}[cc][cc]{\scalebox{0.5}{0.9}}
\psfrag{0.8}[cc][cc]{\scalebox{0.5}{0.8}}
\psfrag{0.7}[cc][cc]{\scalebox{0.5}{0.7}}
\psfrag{0.6}[cc][cc]{\scalebox{0.5}{0.6}}
\psfrag{0.5}[cc][cc]{\scalebox{0.5}{0.5}}
\psfrag{0.4}[cc][cc]{\scalebox{0.5}{0.4}}
\psfrag{0.3}[cc][cc]{\scalebox{0.5}{0.3}}
\psfrag{0.2}[cc][cc]{\scalebox{0.5}{0.2}}
\psfrag{0.1}[cc][cc]{\scalebox{0.5}{0.1}}
\psfrag{0}[cc][rc]{\scalebox{0.5}{0}}

\psfrag{10}[cc][rc]{\scalebox{0.6}{$10$}}
\psfrag{1}[cc][rc]{\scalebox{0.5}{\, 1}}
\psfrag{2}[cc][rc]{\scalebox{0.5}{\, 2}}
\psfrag{3}[cc][rc]{\scalebox{0.5}{\, 3}}
\psfrag{4}[cc][rc]{\scalebox{0.5}{\, 4}}

\psfrag{y}[bc][bc]{\scalebox{0.7}{$\widehat T_{j,j}(h^{(j)}_{\tau_k};S_2)$}}
\psfrag{x}[cc][cc]{\scalebox{0.7}{Number of accepted samples $k$}}

\includegraphics[width=0.48\linewidth]{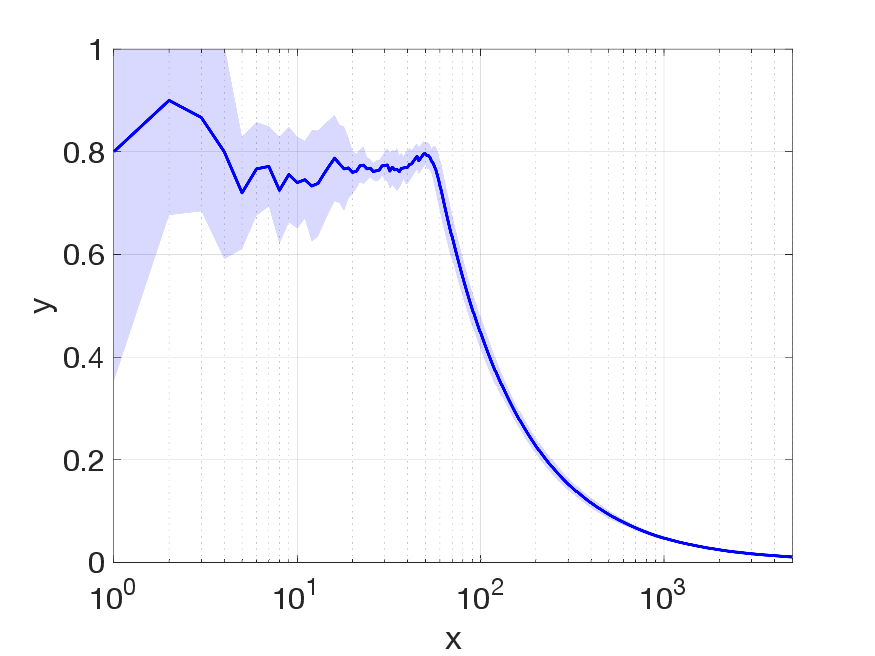}

}%
\subfigure[CIFAR-10 using 60,000 samples.]{

\label{fig:n_plus_thr_uniform_20}

\psfrag{1}[cc][rc]{\scalebox{0.5}{1}}
\psfrag{0.9}[cc][cc]{\scalebox{0.5}{0.9}}
\psfrag{0.8}[cc][cc]{\scalebox{0.5}{0.8}}
\psfrag{0.7}[cc][cc]{\scalebox{0.5}{0.7}}
\psfrag{0.6}[cc][cc]{\scalebox{0.5}{0.6}}
\psfrag{0.5}[cc][cc]{\scalebox{0.5}{0.5}}
\psfrag{0.4}[cc][cc]{\scalebox{0.5}{0.4}}
\psfrag{0.3}[cc][cc]{\scalebox{0.5}{0.3}}
\psfrag{0.2}[cc][cc]{\scalebox{0.5}{0.2}}
\psfrag{0.1}[cc][cc]{\scalebox{0.5}{0.1}}
\psfrag{0}[cc][rc]{\scalebox{0.5}{0}}

\psfrag{10}[cc][rc]{\scalebox{0.6}{$10$}}
\psfrag{1}[cc][rc]{\scalebox{0.5}{\, 1}}
\psfrag{2}[cc][rc]{\scalebox{0.5}{\, 2}}
\psfrag{3}[cc][rc]{\scalebox{0.5}{\, 3}}
\psfrag{4}[cc][rc]{\scalebox{0.5}{\, 4}}

\psfrag{y}[bc][bc]{\scalebox{0.7}{$\widehat T_{j,j}(h^{(j)}_{\tau_k};S_2)$}}
\psfrag{x}[cc][cc]{\scalebox{0.7}{Number of accepted samples $k$}}

  \includegraphics[width=0.48\linewidth]{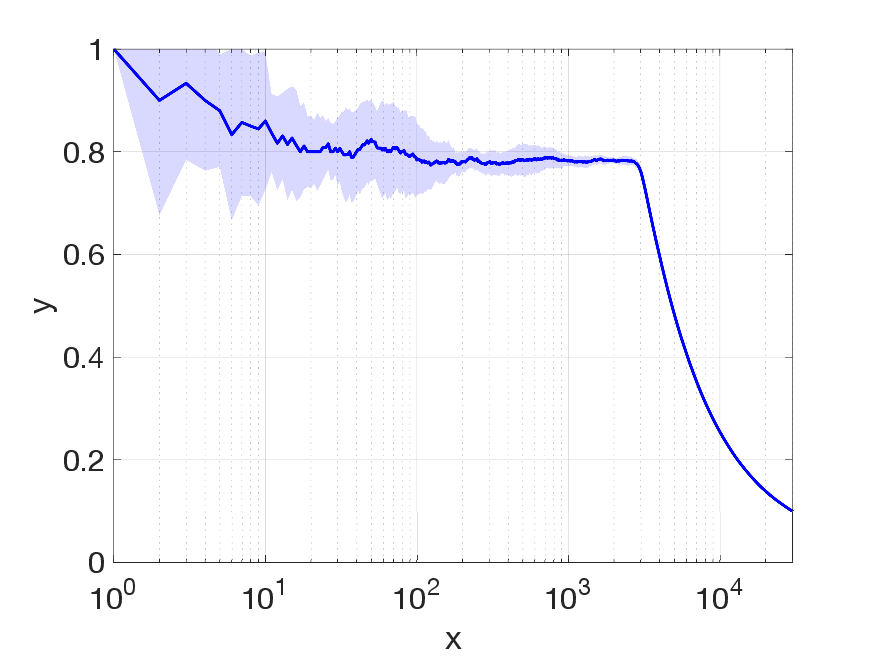}

}%
\caption{The value of the plateau is close to $T_{j,j}=0.8$ and the noise is uniform. Moreover, with fewer samples the plateau is shorter and noisier.}
\label{fig:comparison_N}
\end{figure*}

For Algorithm~\ref{alg:cost_dependent}, we perform a similar analysis by plotting, for $c\in\{0.05,0.1,\ldots,0.95\}$, $(c,\widehat T_{j,j}( h_{c}^{(j)};S_2))$ (blue line) and the number of accepted samples is plotted $(c,\#_{m_2}( h_c^{(j)} ))$ (red line) since the number of accepted samples $\#_{m_2}( h_c^{(j)} )$ typically increases monotonically with $c$.
In Figure~\ref{fig:N_plus_cost} we plot the corresponding graphs for the Satellite dataset.

\begin{figure}[H]

\psfrag{1}[cc][rc]{\scalebox{0.5}{ $ $}}
\psfrag{0.9}[cc][cc]{\scalebox{0.5}{0.9}}
\psfrag{0.8}[cc][cc]{\scalebox{0.5}{0.8}}
\psfrag{0.7}[cc][cc]{\scalebox{0.5}{0.7}}
\psfrag{0.6}[cc][cc]{\scalebox{0.5}{0.6}}
\psfrag{0.5}[cc][cc]{\scalebox{0.5}{0.5}}
\psfrag{0.4}[cc][cc]{\scalebox{0.5}{0.4}}
\psfrag{0.3}[cc][cc]{\scalebox{0.5}{0.3}}
\psfrag{0.2}[cc][cc]{\scalebox{0.5}{0.2}}
\psfrag{0.1}[cc][cc]{\scalebox{0.5}{0.1}}
\psfrag{0}[cc][cc]{\scalebox{0.5}{ 0}}

\psfrag{1}[cc][rc]{\scalebox{0.5}{1}}
\psfrag{2}[cc][rc]{\scalebox{0.5}{2}}
\psfrag{3}[cc][rc]{\scalebox{0.5}{3}}
\psfrag{4}[cc][rc]{\scalebox{0.5}{4}}

\psfrag{500}[cc][cc]{\scalebox{0.5}{500}}
\psfrag{1000}[cc][cc]{\scalebox{0.5}{1000}}
\psfrag{1500}[cc][cc]{\scalebox{0.5}{1500}}
\psfrag{2000}[cc][cc]{\scalebox{0.5}{2000}}
\psfrag{2500}[cc][cc]{\scalebox{0.5}{2500}}
\psfrag{3000}[cc][cc]{\scalebox{0.5}{3000}}

\psfrag{y1}[bc][bc]{\scalebox{0.7}{$\widehat T_{j,j}(h^{(j)}_{c};S_2)$}}
\psfrag{y2}[cc][cc]{\scalebox{0.7}{$\#_{m_2}( h_c^{(j)} )$}}

\psfrag{c}[cc][cc]{\scalebox{0.7}{$c$}}

\centering
  \includegraphics[width=0.48\linewidth]{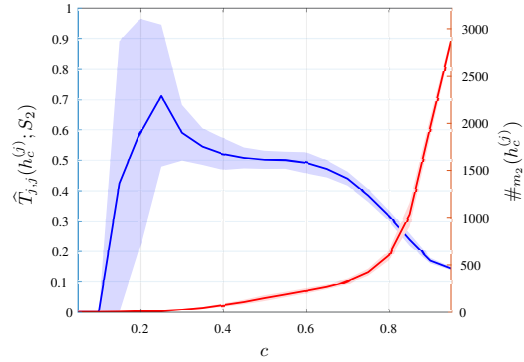}

    \caption{The value of the plateau is close to $T_{j,j}=0.5$ and the noise is uniform.}

  \label{fig:N_plus_cost}
\end{figure}
\vspace{-1.5em}
Again, choosing the highest number of accepted samples inside the plateau ensures a low variance for the estimator while guaranteeing that the bias is small.

The existence and variability of the plateau is determined by the quality of the learned score functions or the binary classifiers in \cref{alg:threshold,alg:cost_dependent}, respectively.

\section{Implementation details and additional experimental results}
\label{app:results}
In the following we provide further implementation details and describe the datasets used in the experimental results. 
Then, we complement the results in the main paper by including a table that compares the performance of the algorithms presented in the paper with the anchor-based method.

\subsection{Implementation details}
\paragraph{Datasets} We evaluate our methodology on four publicly available benchmark datasets from the UCI repository \cite{uci} and Kaggle website \url{https://www.kaggle.com}.
The main characteristics of the datasets used are summarized below.
\begin{itemize}
\item \textbf{Letter:} 20,000 samples, 26 classes, 16 features.
\item \textbf{Satellite:} 6,435 samples, 6 classes, 36 features.
\item \textbf{MNIST:} 70,000 samples, 10 classes, 784 features \cite{mnist}.
\item \textbf{CIFAR-10:} 60,000 samples, 10 classes, 3,072 features \cite{cifar10}.
\end{itemize}

\paragraph{Learned binary classifiers} For the tabular datasets (Letter and Satellite), we employ Random Forests with 200 maximum splits, and a minimum leaf size of 5.
For MNIST, we utilize a multilayer perceptron with two hidden layers of sizes 128 and 64, trained with a regularization parameter of 0.001. 
For CIFAR-10, we perform logistic regression using features extracted from the final layer of a pre-trained 
ConvNeXt based vision foundation model \cite{simeoni2025dinov3}.

\paragraph{Label noise} We generate the noise synthetically, following standard practices in the literature \cite{li21,yong23}. Specifically, for a noise ratio $p \in (0,1)$, we consider two deterministic noise regimes.
\begin{itemize}
    \item $p$-uniform noise: The transition matrix $\T$ is defined such that for every column $j \in \mathcal{Y}$
    \begin{align*}
    T_{j,j} = 1-p, \; T_{i,j} = \frac{p}{|\set{Y}|-1}, \quad \forall i \neq j.
\end{align*}
    \item $p$-flip noise: $T_{1,1}=1-p$, $T_{2,1}=p$ and for every $j\in\{2,3,\ldots,|\set{Y}|\}$,
    \begin{align*}
    T_{j,j} = 1-p, \; T_{j-1,j} = p.
    \end{align*}
\end{itemize}
For example, for \cref{fig:letter} we consider 0.5-uniform noise
and for \cref{fig:MNIST} 0.2-label flip.

\paragraph{Implementation of the anchor-based method} We implement the anchor-based method described in Section~\ref{sec:related_work} by utilizing the 97th percentile of the noisy class-posterior rather than the maximum value, as suggested in the original work \cite{patrini17}. 

\paragraph{Implementation of the proposed algorithms} For \cref{alg:threshold,alg:cost_dependent} we split the samples into two sets of equal size ($m_1 = m_2 = n/2$).

\paragraph{Implementation of the oracle} The oracle for the plots in Section~\ref{sec:experiments} is implemented as follows. First, a feature representation $\varphi(x)\in\R^D$ is obtained from a network trained on clean labels, with a bias coordinate appended to the representation.
Specifically, for the Letter dataset, we use the activations of the final hidden layer of the clean-data MLP (with two hidden layers of sizes 128 and 64), while for MNIST we use the pretrained ConvNeXt representations \cite{simeoni2025dinov3}. 
Given these features, we parametrize the clean class-posterior as $\widehat\eta(x;\mu)=\operatorname{softmax}(\varphi(x)^\up{T} \mu)$ where $\mu\in\R^D$.
Then, we jointly estimate the classifier parameters and transition matrix by minimizing the negative log-likelihood
\begin{align*}
\min_{\mu,\widehat \T}
-\sum_{k=1}^n
\log
\left[
\sum_{j\in\set{Y}}
\widehat T_{\widetilde y_k,j}
\widehat\eta_j(x_k;\mu)
\right]
\end{align*}
subject to $\widehat\T$ being column-stochastic. 
The optimization is initialized with $\mu$ equal to the all-ones vector and $\widehat \T$ at the ground-truth transition matrix $\T$.

\paragraph{Number of repetitions} The results presented at Section~\ref{sec:experiments} (Figure~\ref{fig:num_result}) are averaged over 50 independent trials for each sample size $n$. 
In Appendix~\ref{app:algorithms}, we report the mean over 100 repetitions, where the shaded regions denote the standard deviation. 
Finally, for the results in Appendix~\ref{app:tables}, we report the mean and standard deviation using the full datasets over 5 repetitions for \cref{tab:uniform_0.2,tab:flip_0.45}, whereas for \cref{tab:random} we consider 20 repetitions.

\subsection{Additional experimental results}
\label{app:tables}

\paragraph{Transition matrix estimation error}

Tables \ref{tab:uniform_0.2} and \ref{tab:flip_0.45} report the Mean Absolute Error (MAE) of the estimated transition matrices under \mbox{0.2-uniform} noise and 0.45-flip noise, respectively. 
In \cref{tab:random}, we introduce noise in a more realistic and challenging way.
Specifically, we consider asymmetric label-noise transition matrices with column-wise uniform noise where the noise parameter of each column follows a different value randomly sampled from an $\text{Unif}(0,1/2)$ distribution.

We compare our proposed algorithms (Algorithm~\ref{alg:threshold} and Algorithm~\ref{alg:cost_dependent}) against the standard anchor-based method across multiple datasets. 
To evaluate how the choice of the underlying model impacts estimation, we employ both highly expressive classifiers (e.g., Neural Networks and Random Forests) and weaker, linear methods (Logistic Regression). 
For the high-dimensional CIFAR-10 dataset, we utilize a strong Neural Network architecture.

\begin{table}[H]
\tiny
\caption{\small MAE of different estimators under 0.2-uniform noise. \label{tab:uniform_0.2}}
\begin{tabular}{l c cc cc cc cc cc}
\toprule
  & \multirow{2}{*}{CIFAR-10} & \multicolumn{2}{c}{MNIST} &\multicolumn{2}{c}{Letter} & \multicolumn{2}{c}{Satellite}  \\
 &     & NN & Logistic reg. & RF & Logistic reg. & RF & Logistic reg. \\
\midrule
Anchor-based  & $.009 \pm .001$ &
$.010 \pm .001$ & 
$.016 \pm .001$ & 
$.044 \pm .000$ & 
$.047 \pm .000$ & 
$.022 \pm .001$ & 
$\mathbf{.067 \pm .023}$ \\
Algorithm~\ref{alg:threshold}  & 
$.008 \pm .001$ &
$.008 \pm .001$ &
$\mathbf{.007 \pm .000}$ & 
$\mathbf{.005 \pm .000}$ & 
$\mathbf{.021 \pm .001}$ & 
$\mathbf{.019 \pm .002}$ & 
$.068 \pm .010$ \\
Algorithm~\ref{alg:cost_dependent} &
$\mathbf{.006 \pm .001}$ & 
$\mathbf{.004 \pm .000}$ & 
$.012 \pm .001$ & 
$.006 \pm .000$ & 
$.029 \pm .001$ & 
$.021 \pm .004$ & 
$.085 \pm .003$ \\
\bottomrule
\end{tabular}
\end{table}

\begin{table}[H]
\tiny
\caption{\small MAE of different estimators under 0.45-flip noise. \label{tab:flip_0.45}}
\begin{tabular}{l c cc cc cc cc cc}
\toprule
 & \multirow{2}{*}{CIFAR-10} & \multicolumn{2}{c}{MNIST} & \multicolumn{2}{c}{Letter} & \multicolumn{2}{c}{Satellite}  \\
 &     & NN & Logistic reg. & RF & Logistic reg. & RF & Logistic reg. \\
\midrule
Anchor-based  & 
$.029 \pm .011$ & 
$.016 \pm .009$ & 
$\mathbf{.023 \pm .003}$ & 
$.046 \pm .001$ & 
$.048 \pm .002$ & 
$\mathbf{.039 \pm .011}$ & 
$.085 \pm .056$ \\
Algorithm~\ref{alg:threshold} & 
$.036 \pm .004$ & 
$.033 \pm .003$ & 
$.033 \pm .004$ & 
$\mathbf{.014 \pm .001}$ & 
$\mathbf{.031 \pm .000}$ & 
$.044 \pm .005$ & 
$\mathbf{.066 \pm .003}$ \\
Algorithm~\ref{alg:cost_dependent} & 
$\mathbf{.015 \pm .005}$ & 
$\mathbf{.008 \pm .003}$ & 
$.037 \pm .006$ & 
$\mathbf{.014 \pm .001}$ & 
$.039 \pm .001$ & 
$.045 \pm .004$ & 
$.083 \pm .003$ \\
\bottomrule
\end{tabular}
\end{table}

\begin{table}[H]
\tiny
\caption{\small MAE of different estimators under random asymmetric uniform noise.\label{tab:random}}

\begin{tabular}{l c cc cc cc cc cc}
\toprule
 & \multirow{2}{*}{CIFAR-10} & \multicolumn{2}{c}{MNIST} & \multicolumn{2}{c}{Letter} & \multicolumn{2}{c}{Satellite}  \\
 &     & NN & Logistic reg. & RF & Logistic reg. & RF & Logistic reg. \\
\midrule
Anchor-based  & 
$.009 \pm .001$ & 
$.011 \pm .001$ & 
$.017 \pm .002$ & 
$.041 \pm .001$ & 
$.045 \pm .001$ & 
$.021 \pm .003$ & 
$.102 \pm .065$ \\
Algorithm~\ref{alg:threshold} & 
$.009 \pm .003$ & 
$.008 \pm .001$ & 
$\mathbf{.008 \pm .001}$& 
$\mathbf{.006 \pm .001}$ & 
$\mathbf{.021 \pm .001}$ & 
$\mathbf{.020 \pm .002}$ & 
$\mathbf{.068 \pm .010}$ \\
Algorithm~\ref{alg:cost_dependent} & 
$\mathbf{.006 \pm .001}$ & 
$\mathbf{.005 \pm .001}$ & 
$.011 \pm .003$ & 
$.007 \pm .000$ & 
$.029 \pm .001$ & 
$.021 \pm .003$ & 
$.089 \pm .011$ \\
\bottomrule
\end{tabular}

\end{table}

\vspace{-1em}
As observed across the three noise regimes, our proposed methods outperform the anchor-based baseline on almost all datasets.
As suggested by theoretical results, the anchor-based estimator suffers fundamentally from the difficulty of accurately estimating the class-posterior pointwise.
In contrast, our results empirically validate that accurate transition matrix estimation does not require full pointwise class-posterior recovery. 
Instead, by framing the problem through one-sided selective classification, our methodology only requires learning a sufficiently capable binary classifier.

Furthermore, comparing Table~\ref{tab:uniform_0.2} and Table~\ref{tab:flip_0.45} illustrates the practical impact of the theoretical constant $C_T$ derived in our finite-sample bounds. 
The 0.45-flip noise setting is significantly more challenging than the 0.2-uniform noise setting, which translates to a larger $C_T$.
In accordance with our theory, this increased difficulty causes all evaluated methods to achieve a higher absolute error under flip noise. 
Nonetheless, our methods, particularly Algorithm~\ref{alg:cost_dependent} when implemented with strong classifiers like NNs on CIFAR-10 and MNIST, demonstrate good performance despite the harsh noise setting.

\cref{tab:random} shows that under a more challenging and realistic setting, the different methods present a similar behavior to the scenario from \cref{tab:uniform_0.2}. 
Furthermore, in this more demanding setting, on every dataset, at least one of the proposed algorithms outperforms the anchor-based baseline.

\paragraph{Effect of the dimensionality}
In \cref{fig:dimension}, we show how the error of the different algorithms increases with the number of dimensions $d$.
Specifically, we apply PCA to the ConvNeXt features of CIFAR-10, reducing to dimension $d\in\{100,250,500,1000\}$, and train a logistic regression classifier on the reduced features under uniform 0.2-noise, with the number of training samples fixed at $n=5000$.

The error of the anchor-based method grows with $d$, consistent with the theoretical discussion in the paper. 
In contrast, the error of \cref{alg:threshold,alg:cost_dependent} remain essentially flat across the tested range of $d$, in line with Theorem~\ref{thm:performance_guarantees}'s error bound, which has no explicit dependence on $d$.

\begin{figure*}[htp]

\centering

\psfrag{0.035}[cc][cc]{\scalebox{0.6}{$0.035$\;}}
\psfrag{0.025}[cc][cc]{\scalebox{0.6}{$0.025$\;}}

\psfrag{0.03}[cc][cc]{\scalebox{0.6}{$0.03$\;}}

\psfrag{0.02}[cc][cc]{\scalebox{0.6}{$0.02$\;}}

\psfrag{0.015}[cc][cc]{\scalebox{0.6}{$0.015$\;}}
\psfrag{0.005}[cc][cc]{\scalebox{0.6}{$0.005$\;}}
\psfrag{0.003}[cc][cc]{\scalebox{0.6}{$0.003$ \; }}

\psfrag{10}[cc][rc]{\scalebox{0.48}{$10$}}
\psfrag{1}[cc][rc]{\scalebox{0.4}{1}}
\psfrag{2}[cc][rc]{\scalebox{0.4}{2}}
\psfrag{3}[cc][rc]{\scalebox{0.4}{3}}
\psfrag{4}[cc][rc]{\scalebox{0.4}{4}}

\psfrag{Anchorbased}[cc][cc]{\scalebox{0.65}{Anchor-based}}

\psfrag{Algorithm11}[cc][cc]{\scalebox{0.65}{Algorithm~\ref{alg:threshold}}}

\psfrag{Algorithm22}[cc][cc]{\scalebox{0.65}{Algorithm~\ref{alg:cost_dependent}}}

\psfrag{y}[bc][bc]{\scalebox{0.7}{MAE}}
\psfrag{x}[tc][cb]{\scalebox{0.7}{Number of dimensions $d$}}

\psfrag{100}[cc][cc]{\scalebox{0.6}{$100$}}

\psfrag{250}[cc][cc]{\scalebox{0.6}{$250$}}
\psfrag{500}[cc][cc]{\scalebox{0.6}{$500$}}

\psfrag{1000}[cc][cc]{\scalebox{0.6}{$1000$}}
\psfrag{50000}[cc][cc]{\scalebox{0.6}{$50000$}}

\psfrag{40000}[cc][cc]{\scalebox{0.6}{$ $}}

  \includegraphics[width=0.47\linewidth]{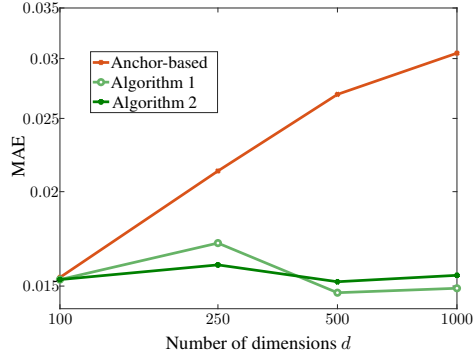}

\caption{The presented methods remain insensitive to the dimension, whereas the error of the anchor-based method increases with the dimension.}
\label{fig:dimension}

\end{figure*}

\paragraph{Running times}

The computational cost of the presented algorithms increases with the number of classes because the number of learned binary classifiers grows linearly with $|\set{Y}|$.
Nevertheless, this increase is not a major concern because each classifier is trained independently for each class, allowing the learning tasks to be run in parallel.

In \cref{fig:time}, we report the median wall-clock training time (in seconds) over 10 executions of each algorithm on the Letter dataset (26 classes) using 26 CPUs.

\cref{alg:threshold}'s runtime is comparable to the anchor-based method's, consistent with each class's binary problem being solved in parallel using one CPU per class. 
\cref{alg:cost_dependent} requires an additional grid search over $\floor{1/\epsilon}=20$ cost values per class ($\epsilon=0.05$), i.e., 520 binary problems in total. 
With class-level parallelization, Algorithm 2's measured runtime is at most 43 times that of Algorithm 1 across the different sample sizes, somewhat above the naive estimate of $20\times$ from grid size alone, reflecting implementation overheads. 
Warm-starting across adjacent grid values could further reduce this gap.

\begin{figure*}[htp]

\centering

\psfrag{250}[cc][cc]{\scalebox{0.6}{$250$\;}}
\psfrag{100}[cc][cc]{\scalebox{0.6}{$100$\;}}

\psfrag{50}[cc][cc]{\scalebox{0.6}{$50$\;}}

\psfrag{25}[cc][cc]{\scalebox{0.6}{$25$\;}}

\psfrag{10}[cc][cc]{\scalebox{0.6}{$10$\;}}
\psfrag{5}[cc][cc]{\scalebox{0.6}{$5$\;}}
\psfrag{2}[cc][tc]{\scalebox{0.6}{$2$\; }}

\psfrag{Anchorbased}[cc][cc]{\scalebox{0.65}{Anchor-based}}

\psfrag{Algorithm11}[cc][cc]{\scalebox{0.65}{Algorithm~\ref{alg:threshold}}}

\psfrag{Algorithm22}[cc][cc]{\scalebox{0.65}{Algorithm~\ref{alg:cost_dependent}}}

\psfrag{y}[bc][bc]{\scalebox{0.7}{Running time [seconds]}}
\psfrag{x}[tc][cb]{\scalebox{0.7}{Number of instances $n$}}

\psfrag{5000}[cc][cc]{\scalebox{0.6}{$5000$}}

\psfrag{7500}[cc][cc]{\scalebox{0.6}{$7500$}}
\psfrag{10000}[cc][cc]{\scalebox{0.6}{$10000$}}

\psfrag{12500}[cc][cc]{\scalebox{0.6}{$12500$}}
\psfrag{15000}[cc][cc]{\scalebox{0.6}{$15000$}}
\psfrag{17500}[cc][cc]{\scalebox{0.6}{$17500$}}

\psfrag{40000}[cc][cc]{\scalebox{0.6}{$ $}}

  \includegraphics[width=0.47\linewidth]{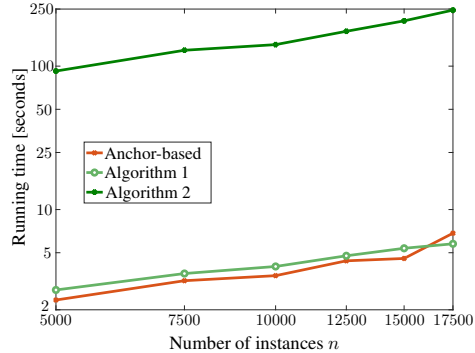}

\caption{The proposed algorithms have low computational overhead in practice as an execution requires at most a few minutes to run.}
\label{fig:time}
\end{figure*}